\documentclass[runningheads]{llncs}
\usepackage[labelfont=bf]{caption}
\usepackage{subcaption}
\usepackage[T1]{fontenc}
\usepackage{tikz}
\usepackage{hyperref}
\usepackage{url}
\usepackage{amsmath}
\usepackage{thm-restate}
\usepackage{graphicx}
\usepackage{algpseudocode}
\usepackage{algorithm}
\usepackage{float}
\usepackage{amssymb}
\usetikzlibrary{shadows} 
\makeatletter\def\Hy@Warning#1{}\makeatother

\begin{document}
\title{Offline Imitation Learning by Controlling the Effective Planning Horizon}
%
%
\author{Hee-Jun Ahn$^*$\inst{1}\orcidID{0009-0007-0389-7270} \and
Seong-Woong Shim$^*$\inst{1}\orcidID{0009-0006-5743-4485} \and
Byung-Jun Lee\inst{1,2}\orcidID{0000-0002-0684-607X}}
\authorrunning{HJ. Ahn et al.}
%
\institute{Korea University, Seoul, South Korea \and
Gauss Labs, Seoul, South Korea\\
\email{\{niwniwniw, ssw030830, byungjunlee\}@korea.ac.kr}}
\maketitle              
\def\thefootnote{*}\footnotetext{These authors contributed equally to this work.}

\begin{abstract}
In offline imitation learning (IL), we generally assume only a handful of expert trajectories and a supplementary offline dataset from suboptimal behaviors to learn the expert policy. While it is now common to minimize the divergence between state-action visitation distributions so that the agent also considers the future consequences of an action, a sampling error in an offline dataset may lead to erroneous estimates of state-action visitations in the offline case. In this paper, we investigate the effect of controlling the effective planning horizon (i.e., reducing the discount factor) as opposed to imposing an explicit regularizer, as previously studied. Unfortunately, it turns out that the existing algorithms suffer from magnified approximation errors when the effective planning horizon is shortened, which results in a significant degradation in performance. We analyze the main cause of the problem and provide the right remedies to correct the algorithm. We show that the corrected algorithm improves on popular imitation learning benchmarks by controlling the effective planning horizon rather than an explicit regularization.
\end{abstract}

\section{Introduction}
\label{introduction}
Imitation learning (IL) is one of the sequential decision making problem settings that aims to solve a task from expert demonstrations instead of explicit notions of utility~\cite{pomerleau1991efficient,ng2000algorithms}. In a standard setting of IL, the agent is only given with a few number of expert trajectories, and it performs the imitation of the expert by interacting with the environment. We denote this setting as online imitation learning since it is possible to query the consequence of action during the learning with the online interactions. On the other hand, it has been recently questioned about the practicality of online interactions, and a setting called offline imitation learning that aims to learn from supplementary suboptimal dataset instead of interactions was proposed~\cite{kim2022lobsdice,kim2021demodice,ma2022versatile}.

As a basic approach for IL, one can come up with behavior cloning (BC)~\cite{pomerleau1991efficient}, which treats the problem as a supervised learning to map state to action based on expert demonstrations. However, BC does not make a use of given information other than expert demonstrations, and suffers from \textit{compounding error} by covariate shift as it gets off track~\cite{ross2011reduction}. To address this issue, previous studies~\cite{ho2016generative,kostrikov2018discriminator,kostrikov2019imitation} have proposed to imitate an expert policy by matching \textit{state-action visitation distributions}, which can be estimated by interacting with the environment. State-action visitation matching allows the agent to learn how to return to the preferred states, and this approach has been successful in a wide range of domains of online IL.

Interestingly, in recent offline IL studies, it is often found that BC is quite competitive~\cite{li2022rethinking,kim2021demodice}, and sometimes it even performs better than visitation distribution matching algorithms depending on the experiment settings. This is because there is an inherent error in estimating visitation distributions from a finite dataset in offline cases, and BC, which completely ignores the dynamics information, does not suffer from this problem. Therefore, unlike the online cases where we can sample from true environment dynamics, BC or visitation distribution matching cannot always be the superior choice over the other in offline cases.

Motivated by this, we consider controlling the effective planning horizon, i.e. the discount factor $\gamma$, for the offline imitation learning problems. The discount factor implies how important the future is compared to the present. For visitation distribution matching algorithms, it can be interpreted as the amount of state-action distribution difference between the learned and expert policy we allow in the future compared to now. While it has been overlooked in the online IL studies since we recover more accurate policy as we less discount the future, it will be advantageous to use an optimal discount factor in offline IL that makes the best trade-offs:
using small $\gamma$ to shorten the effective planning horizon helps to become robust to the errors in the inferred dynamics~\cite{janner2019trust} (which also allows avoiding explicit regularization for robustness as in previous studies), whereas using large $\gamma$ makes an agent less prone to the compounding errors by considering longer consequences of training.

In this paper, we start by formally analyzing the performance trade-off of a discount factor in offline IL. It turns out that, however, recently proposed offline IL algorithms show pathological behavior when the discount factor is lowered and naively controlling the discount factor does not lead to a performance gain. It is because the error from the approximations they share depends on the discount factor, and it becomes revealed and maximized as we lower $\gamma$. To this end, we propose a simple technique called Inverse Geometric Initial state sampling (IGI) to address the problem. We show that IGI enables the algorithm to properly learn in the low discount factor settings, and by tuning the discount factor, the offline IL agent with IGI outperforms previous state-of-the-art algorithms with explicit regularizations.
\section{Background}
\label{Background}

\subsection{Markov Decision Process (MDP)}
\label{mdp}
We consider environments modeled as a Markov Decision Process (MDP), which is defined by a tuple $\mathcal{M}$ = ($\mathcal{S},\,\mathcal{A},\,P,\,p_0,\,R,\,\gamma$). Here, $\mathcal{S}$ is the state space, and $\mathcal{A}$ is the action space. $P$ and $p_0$ represent the dynamics and initial state distribution, respectively. $R$ is the reward function which is assumed to be bounded in $[0, R_{\max}]$, and $\gamma\in[0,1)$ is the discount factor. A policy $\pi(\cdot\vert s)$ determines the probability of agent's action in a state $s$. The goal of an agent is to maximize the sum of reward discounted by $\gamma$. The state-action visitation distribution $d^\pi(s,a)$, which is induced by a policy $\pi$, is defined as
\begin{equation*}
   d^\pi(s,a) = (1-\gamma)\sum_{t=0}^\infty\gamma^t\textrm{Pr}(s_t=s,\,a_t=a
   \,\vert \,p_0,\,P,\,\pi).
\end{equation*}
where $s_0 \sim p_0(\cdot)$ and $a_t \sim \pi(\cdot\vert s_t ), s_{t+1}\sim P(\cdot\vert s_t,a_t)$ for 
all time step $t$. We consider an offline imitation learning problem where a small number of expert demonstrations ($\mathcal{D}^E$) and a suboptimal dataset ($\mathcal{D}^O$) consisting of transitions $(s, a, s^\prime)$ are given. We denote total dataset $\mathcal{D}^D=\mathcal{D}^E\cup \mathcal{D}^O$ as a union of two datasets. Especially, we assume that $\mathcal{D}^E$ is sampled by following the expert policy $\pi^E$ in underlying MDP $\mathcal{M}$. We denote empirical distribution of $\mathcal{D}^E$ and $\mathcal{D}^D$ as $E(s,a)$ and $D(s,a)$ respectively. On the other hand, we denote the state-action distribution of $\pi^E$ as $d^E(s,a)$.

\subsection{Imitation Learning via state-action visitation matching}
\label{imitation learning}
Imitation learning (IL) aims to train an agent that mimics the expert based on the expert demonstrations. In addition to demonstrations, supplementary information on environment dynamics is given, by directly interacting with environment (online IL) or by a dataset of suboptimal behaviors (offline IL). 
One simplest and most popular approach to imitate an expert is behavior cloning (BC), which treats IL as a supervised learning problem, ignoring any of supplementary information~\cite{pomerleau1991efficient}. The objective of BC can be represented as follows:
\begin{equation}
\label{BC}
\min_{\pi} -\mathbb{E}_{(s,a)\sim E}\left[\log{\pi(a\vert s)} \right].
\end{equation}
However, it is well known that the error of obtained policy induced by covariate shift compounds over time, and leads to an eventual failure unless $\mathcal{D}^E$ is large enough~\cite{ross2011reduction}. 

To address the weakness of BC, a state-action visitation distribution matching objective is now widely adopted, which minimizes the divergence between $d^\pi$ and $d^E$~\cite{kostrikov2019imitation,ke2020imitation}. The distribution matching objective can be presented as follows:
\begin{equation}
\label{distribution matching}
    \max_{\pi} -D_{\text{KL}}\left( d^\pi\Vert d^E\right)=\mathbb{E}_{(s,a)\sim d^\pi}\left[\log{\frac{d^E(s,a)}{d^\pi(s,a)}} \right].
\end{equation}
Note that the objective can be alternatively interpreted as a reinforcement learning (RL) problem:
\begin{equation}
\label{distribution matching to RL}
    \max_{\pi} -D_{\text{KL}}\left( d^\pi\Vert d^E\right)=(1-\gamma)\cdot\mathbb{E}_{\begin{subarray}{1}s_0\sim p_0,\\ a_t\sim \pi,\\ s_{t+1}\sim P\end{subarray}}\left[\sum_{t=0}^\infty\gamma^t\log{\frac{d^E(s_t,a_t)}{d^\pi(s_t,a_t)}} \right].
\end{equation}
We can observe that the objective (\ref{distribution matching to RL}) can be seen as an RL problem with a reward $r=\log{\frac{d^E}{d^\pi}}$. When online interactions are allowed, we can sample from $d^\pi$ by executing policy $\pi$, and \cite{ho2016generative} proposed to estimate the newly defined reward by learning the discriminator using samples from $d^E$ and $d^\pi$:
\begin{equation}
\label{discriminator}
    \max_{c:S\times A\rightarrow(0,1)}\mathbb{E}_{(s,a)\sim d^E}\left[\log{c(s,a)} \right]+\mathbb{E}_{(s,a)\sim d^\pi}\left[\log{(1-c(s,a))} \right],
\end{equation}
where the optimal discriminator $c^*(s,a)$ can be used to recover the reward as $\log{c^*(s,a)}-\log{(1-c^*(s,a))}=\log{\frac{d^E(s,a)}{d^\pi(s,a)}}$. On the contrary, in a fully offline setting we consider, the interaction with the environment to receive samples from $d^\pi$ is not allowed. A common choice in this case is to train a discriminator that can distinguish between the expert demonstrations and a supplementary suboptimal dataset~\cite{kim2021demodice,kim2022lobsdice,ma2022versatile}, and is explained in detail in Section \ref{main objective}. 

\section{Offline IL via state-action visitation matching}
\label{discount factor in offline RL}
In this section, we first derive a practical objective that can be used to optimize the distribution matching objective (\ref{distribution matching}) in the offline IL case. After the derivation of an objective that is used throughout the paper, we provide an analysis of the effect of the discount factor on the error bound of offline IL. In particular, we show that there is a trade-off when controlling the discount factor, unlike the online case.
\subsection{Derivation of an offline IL objective}
\label{main objective}
We start from the widely used visitation distribution matching objective (\ref{distribution matching}) as suggested in \cite{ho2016generative,torabi2018generative}. 
Rather than optimizing for the policy $\pi$, we follow the recent approaches that optimize directly for the visitation distribution $d^\pi$~\cite{lee2021optidice,kim2021demodice}. To optimize for $d^\pi$, we need to ensure that the solution we get is a valid visitation distribution. By making Bellman flow constraints and normalization constraints explicit, we have: 
\begin{align}
\label{3.2_1}
	\max_{d^\pi}\;&-D_{KL}(d^\pi\Vert d^E) \\
	\label{3.2_2}
    	\textrm{s.t}\; &\mathcal B_*d^\pi(s,a)=(1-\gamma)p_0(s)+\gamma\mathcal P_* d^\pi(s,a)\;\,\forall s,\\ \label{3.2_3}
        &\sum_{s,a}d^\pi(s,a)=1\,,\;d^\pi(s,a) \geq 0 \;\,\forall s,a,
\end{align}
where $\mathcal{B_*}d^\pi(s,a)=\sum_a d^\pi(s,a)$ and $\mathcal{P}_*d^\pi(s,a)=\sum_{\Bar{s},\Bar{a}}P(s\vert\allowbreak\Bar{s},\Bar{a})d^\pi(\Bar{s},\Bar{a})$ are a marginalization and an expectation over the previous state-actions, respectively. It can be easily seen that the following equalities hold by interchanging the order of summations:
\begin{equation*}
\begin{split}
    &\sum_s \nu(s)\mathcal{P}_*d^\pi(s,a) = \sum_{s,a}d^\pi(s,a) \mathcal{P}\nu(s),\;\textrm{where}\;\mathcal{P}\nu(s)=\sum_{s'} P(s'\vert s,a)\nu(s')\\
    &\sum_s \nu(s)\mathcal{B}_*d^\pi(s,a) = \sum_{s,a}d^\pi(s,a) \mathcal{B}\nu(s),\;\textrm{where}\;\mathcal{B}\nu(s)=\nu(s).
\end{split}
\end{equation*}
The Lagrangian of the constrained problem (\ref{3.2_1}-\ref{3.2_3}) is
\begin{align}
\begin{split}
\label{3.2_4}
    \max_{d^\pi\geq0}\min_{\nu,\lambda}&-\mathbb{E}_{s,a\sim d^\pi}\left[\log\frac{d^\pi(s,a)}{d^E(s,a)}\right] +\lambda\left[\sum_{s,a}d^\pi(s,a)-1\right]\\
    &+\sum_s \nu(s)((1-\gamma)p_0(s)+\gamma\mathcal P_*d^\pi(s,a)-\mathcal{B}_*d^\pi(s,a))
\end{split}
\end{align}
where $\lambda$ and $\nu(s)$ are Lagrange multipliers. Based on the relationships shown above, we can rewrite Equation (\ref{3.2_4}), where the summation is computed w.r.t. $d^\pi(s,a)$:
\begin{align}
\label{3.2_5}
\max_{d^\pi\geq0}\min\limits_{\nu,\lambda}(1-\gamma)\mathbb{E}_{s\sim p_0}[\nu(s)]+\sum_{s,a}d^\pi(s,a)\left[e_\nu(s)+\lambda-\log\frac{d^\pi(s,a)}{d^E(s,a)}\right]-\lambda.
\end{align}
where $e_\nu(s)=\gamma \mathcal{P}\nu(s)-\mathcal{B}\nu(s)$. In the second term of Equation (\ref{3.2_5}), to derive a practical algorithm, we want to make an expectation to be performed on the total dataset distribution $D(s,a)$. For this, we use the importance sampling with importance weight $\zeta(s,a)=\frac{d^\pi(s,a)}{D(s,a)}$. Then, the objective becomes
\begin{align}
\begin{split}
\label{3.2_6}
    \max_{\zeta\geq0}\min_{\nu,\lambda}
    &(1-\gamma)\mathbb{E}_{p_0}[\nu(s)]+\mathbb{E}_{D}\left[\zeta(s,a)\left(A_\nu(s,a)+\lambda-\log\zeta(s,a)\right)\right]-\lambda.\\&=:\mathcal{L}(\zeta,\nu,\lambda)
\end{split}
\end{align}
Here, $A_\nu(s,a) = e_\nu(s)+\log\frac{d^E(s,a)}{D(s,a)}$, which can be interpreted as an advantage when the action $a$ is executed in the state $s$ if we define the reward as a log ratio of distributions $r(s,a)=\log\frac{d^E(s,a)}{D(s,a)}$. Such a log ratio can be obtained by training a discriminator $c(s,a)$ through the following objective:
\begin{align}
\label{3.2_11}
    \max_{c:S\times A\rightarrow[0,1]} \;\mathbb{E}_{(s,a)\sim d^E}[\log c(s,a)]+\mathbb{E}_{(s,a)\sim D}[\log (1-c(s,a))],
\end{align}
such that the optimal discriminator $c^*(s,a)$ satisfies $\log\frac{d^E(s,a)}{D(s,a)} = \log \left(\frac{c^*(s,a)}{1-c^*(s,a)}\right)$. The advantage $A_\nu(s,a)$ can then be estimated based on the discriminator $c$ being trained. We can further simplify the optimization problem (\ref{3.2_6}) by noting that the strong duality holds, which enables us to change the order of optimization without affecting the optimal value. Closed-form solution of maximization of Equation (\ref{3.2_6}) can be obtained by solving first-order optimality condition $\frac{\partial\mathcal{L}(\zeta^*,\nu,\lambda)}{\partial\zeta} = 0$, which gives $\zeta^*=\exp{(A_\nu(s,a)+\lambda-1)}$. 

Similarly, $\lambda^*=-\log[\mathbb{E}_{(s,a)\sim D}\exp{(A_\nu(s,a)-1)}]$ can be obtained. By substituting in $\zeta^*$ and $\lambda^*$ to (\ref{3.2_6}), we derive the final objective:
\begin{align}
\label{3.2_9}
    \min_\nu \;(1-\gamma)\mathbb{E}_{p_0}[\nu(s)]+\log\mathbb{E}_{(s,a)\sim D}\left[\exp{\left(A_\nu(s,a)\right)}\right]=:\mathcal{L}(\zeta^*,\nu,\lambda^*).
\end{align}
Based on the optimized $\nu^*$, we can obtain $\zeta^*=\text{softmax}(A_\nu(s,a)-1) $. Using this, policy can be extracted through the following objective:
\begin{align}
\label{3.2_10}
    \min_\pi \;\mathbb{E}_{(s,a)\sim d^{\pi^*}}[\log\pi(a\vert s)]=\mathbb{E}_{(s,a)\sim D}\left[\zeta^*\log\pi(a\vert s)\right].
\end{align}
Note that the objective (\ref{3.2_9}) can be seen as a special case of DemoDICE~\cite{kim2021demodice} and SMODICE~\cite{ma2022versatile}, and it can be recovered by setting $\alpha=0$ in DemoDICE (no explicit regularization toward suboptimal dataset) or by setting $f$-divergence to be KL-divergence in SMODICE. While the algorithm in this paper is based on KL-divergence minimization of visitation distributions, note that it can be trivially extended to any $f$-divergence minimization case.

\subsection{Trade-off between two distinct effects in offline IL by discount factor}
\label{trade off by discount factor}
Now we show that there is indeed a trade-off when controlling the discount factor $\gamma$ in the offline setting. Intuitively, an offline agent can only receive supplementary information about environment dynamics through limited finite demonstrations, and the error of estimated dynamics is inevitable.
Therefore, we can expect that learning with a long planning horizon increases the risk of compounding estimation error. On the other hand, the trained agent is typically evaluated in a non-discounted environment by measuring the average reward it gets, and excessively lowering the discount factor for training will make a large train-test discrepancy.
The theorem below is a formal analysis backing up the argument, showing an error bound of the imitated policy with respect to a discount factor.

\begin{restatable}{theorem}{bound}
\label{theorem1}
Let $P$ an underlying transition dynamics and $\widehat{P}$ an estimated transition dynamics. $\gamma$ is a discount factor used for evaluating the policy and $\hat{\gamma}$ is a discount factor used for training the policy where $\hat{\gamma}\le\gamma$. Let $d^\pi_{P, \gamma}$ a state-action visitation distribution induced by $\pi$ under the dynamics $P$ using $\gamma$, and $d^E_{P, \gamma}$ a state-action visitation distribution induced by $\pi^E$ under the dynamics $P$ using $\gamma$. Assume that the reward is bounded in $[0, R_{\max}]$.
Then, the error bound for imitated policy $\pi$ is
\begin{align}
\begin{split}
\label{theorem1_1}
    \left\vert \mathbb{E}_{d^\pi_{P,\gamma}}[ r(s,a) ]-\mathbb{E}_{d^{E}_{P,\gamma}}[ r(s,a) ]\right\vert \le
    \frac{2R_{\max}}{1-\hat{\gamma}}\left(\frac{\gamma-\hat{\gamma}}{1-\gamma}
    +\frac{\hat{\gamma}\epsilon_P}{1-\hat{\gamma}}+\frac{\epsilon_\pi}{1-\hat{\gamma}}\right),
\end{split}
\end{align}
where, \\
$\epsilon_P=\mathbb{E}_{d^\pi_{\widehat{P},\hat{\gamma}}}\left[D_{TV}(\widehat{P}\Vert P)\right]+\mathbb{E}_{d^E_{\widehat{P},\hat{\gamma}}}\left[D_{TV}(\widehat{P}\Vert P)\right]$, and $\epsilon_{\pi}=\mathbb{E}_{d^\pi_{\widehat{P},\hat{\gamma}}}\left[ D_{TV}(\pi\Vert\pi^E)\right]$.
\end{restatable}
The proof is deferred to Appendix \ref{proof of theorem 1}. The first term reflects the train-test discrepancy caused by optimizing policy using $\hat{\gamma}$ instead of the discount factor $\gamma$ that is used for evaluation, i.e. in the LHS of inequality (\ref{theorem1_1}), and is minimized as we choose $\hat{\gamma}$ close to $\gamma$. The second term represents the effect of the error on estimating dynamics $\epsilon_P$. It can be observed that this term becomes smaller if the effective planning horizon is shortened by using a lower $\hat{\gamma}$. From the definition of $\epsilon_P$, it is the accumulation of step-wise dynamics estimation error over trajectories of both $\pi$ and $\pi^E$, and there will always be an irreducible amount of $\epsilon_P$ in the offline case. The last term stands for the policy difference between $\pi$ and $\pi^E$ over the estimated dynamics and will be minimized if imitation learning is done properly to maximize the objective (\ref{distribution matching}).

In short, the error of the imitated policy will be mainly dependent on the first two terms of (\ref{theorem1_1}), which have opposite dependencies to the discount factor for training $\hat{\gamma}$. This result supports our argument about the trade-off between two distinct effects by discount factor during the offline training and shows the possibility that an offline IL algorithm may benefit from choosing $\hat{\gamma}$ that is smaller than $\gamma$ that we use to evaluate.

\section{Controlling the discount factor in offline imitation learning}
\label{4}

As suggested by Theorem \ref{theorem1}, we expect that there would be a benefit from controlling the discount factor $\gamma$ that is used for training. The algorithm derivation is based on a visitation distribution matching~(\ref{3.2_1}-\ref{3.2_9}), as well as those of previous studies, which appears independent of and applicable to any choice of the discount factor in principle.
However, it turns out that naively lowering $\gamma$ in previously suggested visitation distribution matching offline IL algorithms~\cite{kim2021demodice,ma2022versatile} only results in a monotonically decreasing performance, as shown in Figure \ref{various_gamma}. The performance reduction of using lower discount factors is significant, and it seems pointless to choose any $\gamma$ below $0.99$ in contrast to our analysis. In this section, we analyze the cause of this pathological behavior and propose a simple scheme that can alleviate it.
\begin{figure}[t]
    \centering
    \includegraphics[width=4.5in]{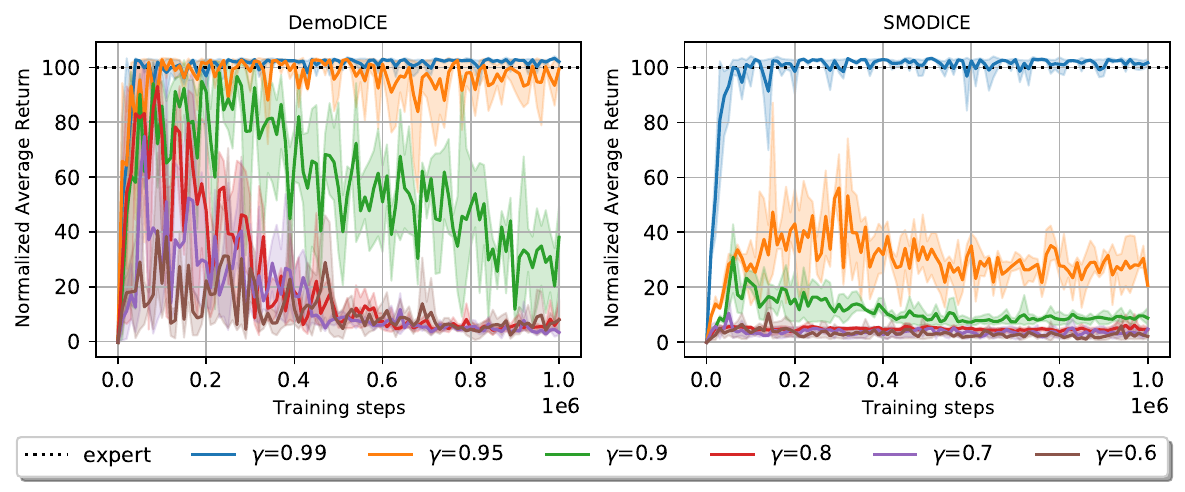}
    \caption{Learning curves of DemoDICE~\cite{kim2021demodice} and SMODICE~\cite{ma2022versatile} over various choices of discount factor $\gamma$. We use $\mathcal{D}^O$ consisting of 400 expert trajectories and 800 random-policy trajectories, and $\mathcal{D}^E$ consisting of 1 expert trajectory. Evaluations are averaged over 3 seeds in a \texttt{Hopper-v2} environment, and they are normalized so that 0 corresponds to the average score of the random-policy dataset, and 100 corresponds to the average score of the expert policy dataset. }
    \label{various_gamma}
\end{figure}
\begin{figure}[t]
\centering
\begin{subfigure}{0.49\textwidth}
    \centering
    \label{fig2_1}
    \begin{tikzpicture}[node distance = 3cm, auto]
    \tikzset{
        cli/.style={circle,minimum size=0.3cm,fill=white,draw,
                    general shadow={fill=gray!60,shadow xshift=1pt,shadow yshift=-1pt}},
            }
    \node[cli] (1) at (0,0) {$s_1$};
    \node[cli] (2) at (2.5,0) {g};
    \node[cli] (0) at (1.25,1.5) {$s_0$};
    \path [-stealth, thick]
        (0) edge [left]node {$a_1=0.6$} (1)
        (1) edge [below] node {$a_1=1$} (2)
        (0) edge [right] node {$a_2=0.4$} (2)
        (2) edge [loop right]  node {$a_1$} ();
    \end{tikzpicture}
    \caption{expert policy}
\end{subfigure}
\hfill
\begin{subfigure}{0.49\textwidth}
    \centering
    \label{fig2_2}
    \begin{tikzpicture}[node distance = 3cm, auto]
    \tikzset{
        cli/.style={circle,minimum size=0.3cm,fill=white,draw,
                    general shadow={fill=gray!60,shadow xshift=1pt,shadow yshift=-1pt}},
            }
    \node[cli] (1) at (0,0) {$s_1$};
    \node[cli] (2) at (2.5,0) {g};
    \node[cli] (0) at (1.25,1.5) {$s_0$};
    \path [-stealth, thick]
        (0) edge [left]node {$a_1=\theta$} (1)
        (1) edge [below] node {$a_1=1$} (2)
        (0) edge [right] node {$a_2=1-\theta$} (2)
        (2) edge [loop right]  node {$a_1$} ();
    \end{tikzpicture}
    \caption{learned policy}
\end{subfigure}
\caption{Toy infinite horizon MDP example with 3 states and 2 actions. All transitions are deterministic and shown with the arrows. $s_0$ is initial state, and $g$ is absorbing state. We indicate the probability of taking an action based on the corresponding policy on the arrows in the figure. (a) represents expert policy and (b) is for learned policy with parameter $\theta$.}
\label{simple mdp}
\end{figure}
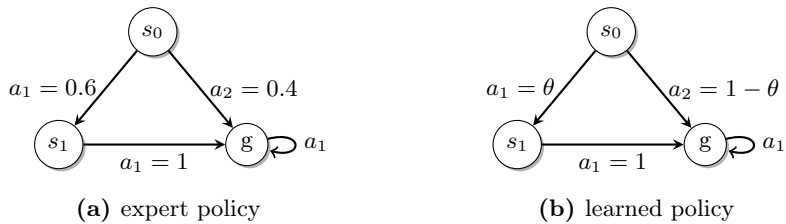
\subsection{Distribution mismatch in training a discriminator}
\label{4.1}
The main problem lies in the fact that while the visitation distribution matching objective (\ref{3.2_1}) requires the estimation of $d^E(s,a)$, which is a discounted state-action distribution induced by the expert policy $\pi^E$, a discriminator $c(s,a)$ is learned to discriminate the empirical distributions of the datasets, i.e. between $D(s,a)$ and $E(s,a)$. Since the empirical distributions are undiscounted, for a state-action pair sample $(s_t,a_t)$ at timestep $t$, $d^E(s,a)$ would weight $\gamma^t$ on this sample compared to $E(s,a)$ that does not weight on this sample, making the two distributions significantly different as we have smaller $\gamma$.

To observe the consequence of having a discriminator trained on empirical distributions, we assume that the reward is represented with an optimal discriminator such that $r(s,a)=\log\frac{E(s,a)}{D(s,a)}$. By applying this to objective (\ref{3.2_6}), it becomes:
\begin{align}
\begin{split}
\label{4.1_1}
    &\max_{d^\pi\geq0}\min_{\nu,\lambda}\mathcal{L}(d^\pi,\nu,\lambda) =(1-\gamma)\mathbb{E}_{s\sim p_0}[\nu(s)]\\&+\sum_{s,a}d^\pi(s,a)\left[\gamma \mathcal{P}\nu(s)-\mathcal{B}\nu(s)+\log\frac{E(s,a)}{D(s,a)}-\log\frac{d^\pi(s,a)}{D(s,a)}+\lambda\right]-\lambda.
\end{split}
\end{align}

By taking the derivation steps backward, we can confirm that using the discriminator trained on empirical distributions is equivalent to solving the visitation distribution matching objective (\ref{3.2_1}-\ref{3.2_3}) except for the objective:
\begin{align}
\begin{split}
\label{4.1_3}
	\max_{d^\pi}\;&-D_{KL}(d^\pi(s,a)\Vert E(s,a)).
\end{split}
\end{align}

That is, by training a discriminator trained on empirical distributions, we actually have matched $d^\pi(s,a)$ and $E(s,a)$ to get a policy $\pi$.
Note that there is a one-to-one correspondence between the state-action visitation distributions and policies~\cite{ho2016generative}. It implies that unless $E(s,a)=d^E(s,a)$, even if $E(s,a)$ is a valid visitation distribution that satisfies Bellman flow constraints, the policy inferred by matching $E(s,a)$ will be different from $d^E$, and thus, $\pi^*\neq \pi_E$. In most cases, $E(s,a)$ would not be a valid visitation distribution, and a policy $\pi$ that gives $D_{KL}(d^\pi(s,a)\Vert E(s,a))=0$ would not exist in general. While this discrepancy between $E$ and $d^E$ is negligible when $\gamma$ is large and the length of trajectories in $\mathcal{D}^E$ is long enough, large error on the imitated policy can be incurred in other cases. We give a simple toy example below.

\subsubsection{Illustrative example} Assume an infinite horizon MDP with 2 states and an absorbing state, and with 2 actions as shown in Figure~\ref{simple mdp}. The expert policy $\pi_E$ has $0.6$ probability of doing $a_1$ in $s_0$, $0.4$ of doing $a_2$ in $s_0$, and $1$ of doing $a_1$ in the other states. Assuming that trajectories are sampled up to timestep 2, the expert demonstration $\mathcal{D}^E$ will be consisting of two kinds of trajectories $\tau_1 = (s_0, a_1, s_1, a_1, g, a_1)$ and $\tau_2 = (s_0, a_2, g, a_1, g, a_1)$ with the ratio of $6:4$. $E(s,a)$ can be calculated by $\frac{\textrm{The number of (s, a) pairs}}{\textrm{Total number of state-action pairs}}$ accordingly. The learned policy is parametrized as $\pi(a_1|s_0)=\theta$, $\pi(a_2|s_0)=1-\theta$, and $\pi(a_1|\cdot)=1$ in other states. By computing $d^\pi$ based on the above, we can express $D_{KL}(d^\pi\Vert E)$ in terms of $\theta$ and $\gamma$. Then, we can minimize $D_{KL}$ to get $\theta^*$, and see if we can recover $\theta^*=0.6$. In this toy example, it turns out that $\theta\neq0.6$ unless $\gamma=0.5$ (detailed derivations can be found in Appendix \ref{proof for simple MDP}). It can be seen in this example that we cannot match two distributions $d^\pi(s,a)=d^\pi(s)\pi(a|s)$, $E(s,a)=E(s)\pi^E(a|s)$ in general, and by minimizing the distribution between state visitation distributions, optimized $\pi$ will be different to $\pi_E$.

\subsection{Inverse geometric initial state sampling}
\label{4.2}
As shown in the previous subsection, we cannot recover $\pi=\pi^E$ unless $d^E=E$, where $E$ is the distribution that the discriminator is trained on. One straightforward solution to this problem is to train the discriminator to distinguish $d^E$ and $E$, eliminating the root cause of the problem. We can simulate the sampling from $d^E$ by weighting each sample in the dataset with respect to their timesteps when they were sampled. For example, we can first sample the timestep $t$ from a geometric distribution $\text{Geom}(1-\gamma)$ and then sampling $(s_t,a_t)$ pair that had been sampled at $t$. This procedure will approximate the sampling from $d^E$ sufficiently well, given that the trajectories stored in $\mathcal{D}^E$ is long enough.

However, training the discriminator to distinguish $d^E$ gives rise to another problem; if we use lower $\gamma$, the samples from $d^E$ will mostly consist of early timestep samples of $\mathcal{D}^E$, since the probability assigned decreases exponentially over timesteps. This results in a significant under-usage of dataset $\mathcal{D}^E$ by throwing away samples of later timesteps, and significantly deteriorates the performance. We need a way to ensure that we get $\pi=\pi^E$ while not hurting the effective number of data.

To this end, we propose to devise a different initial state distribution other than the actual $p_0$ to satisfy both conditions. Note that even if we change the initial distribution, the optimality of policies are not affected as mentioned in~\cite{kostrikov2019imitation}. $d^\pi$ and $d^E$ that $\pi$ and $\pi_E$ induces will be different, but one-to-one correspondence does still hold that matching visitation distributions ensure $\pi=\pi_E$ with arbitrary initial state distribution.

In particular, if we use an initial distribution that makes $d^E=E$, i.e. flattens the visitation distribution to look like an undiscounted distribution, recovering $\pi=\pi_E$ will be guaranteed even if we use the discriminator that learns from the empirical distributions. Note that sampling from $d^E$ corresponds to sample timestep $t$ from $\text{Geom}(1-\gamma)$ and choose $(s_t,a_t)$ that had been sampled at $t$. Hence, if we use a modified initial distribution $\tilde{p}_0$ that samples from all the timesteps in $\mathcal{D}^E$ with weights inversely proportional to $\text{Geom}(1-\gamma)$, the resultant $\tilde{d}^E$ will have uniform weights regarding timesteps and therefore $\tilde{d}^E=E$. We call this modified initial state sampling method as Inverse Geometric Initial state sampling (IGI). In practice, there are a few more things to consider determining the weightings for each possible initial state samples; e.g. due to datasets being truncated at certain timesteps and trajectories that are terminated early. Nevertheless, it is possible to set up a system of linear equations to obtain the required weights for each sample to ensure $\tilde{d}^E=E$. Details on how to set up a linear system is shown in Appendix \ref{IGI}.

\begin{figure}[t!]
    \centering
    \includegraphics[height=1.5in]{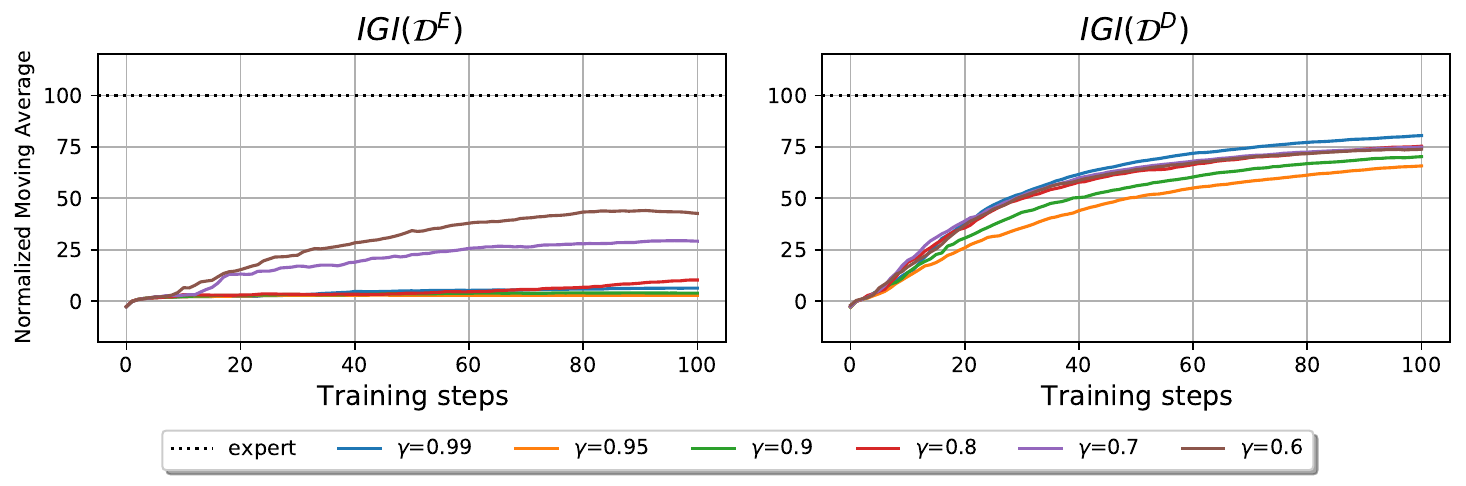}
    \caption{Learning curves of our method in \texttt{HalfCheetah-v2}. Here, suboptimal dataset is consisted of 100 expert trajectories and 800 random-policy trajectories. We applied moving average with 3 seeds. The plot on the left is the result of sampling $\tilde{p}_0(s\vert t)$ from expert dataset, and the right is the result of sampling from total dataset.}
    \label{irucompare}
\end{figure}

Unfortunately, it turns out that the exact IGI that ensures $\tilde{d}^E=E$ can be too restrictive. Following the recent practices of using a single expert policy for $\mathcal{D}^E$ in evaluating imitation performances, only the state-action pairs in the single trajectory can be sampled as an initial state, which cripples the diversity. Furthermore, we found that sampling the initial states only from $\mathcal{D}^E$ also limited the opportunity of our agent to learn how to act outside $\mathcal{D}^E$, incurring the covariate shift and compounding error problems that BC has suffered. On the other hand, using IGI to sample from the total dataset $\mathcal{D}^D$ did not exhibit such a problem, while alleviating the discrepancy between $\tilde{d}^E$ and $E$. In Figure \ref{irucompare}, we compare the performance of two options: IGI($\mathcal{D}^E$) and IGI($\mathcal{D}^D$). It is clearly visible that the imitation performance of IGI($\mathcal{D}^E$) is limited, whereas IGI($\mathcal{D}^D$) performs robustly over a wide set of $\gamma$s. As a consequence, we used IGI with samples from $\mathcal{D}^D$ in the following experiments.

\section{Related work}
\label{related work}
\subsubsection{Controlling the effective planning horizon} In reinforcement learning (RL), the discount factor effectively controls the amount of forward planning the agent considers making a decision. \cite{petrik2008biasing} has shown that the approximation error bound incurred in an inaccurate model can be tightened by using a low discount factor. \cite{jiang2015dependence} has shown the similar result, but they have analyzed in terms of the complexity of the class of policies. They show that there is a trade-off between the complexity of the policy space and the error due to the approximated model according to the effective planning horizon, which can controlled by a discount factor. Recently, \cite{hu2022role} analyzed the role of discount factor from the perspective of offline RL. In addition to the trade-off relationship between optimality and sample efficiency, they show theoretically and empirically that the low discount factor can also be seen as a model-based pessimism. On the other hand, there has not been a study on the effect of discount factor on offline imitation learning, up to our knowledge.
\subsubsection{Offline imitation learning by leveraging the duality of RL} Imitation learning (IL) aims to mimic the expert policy by using the expert demonstrations and the online interactions, but offline supplementary dataset is given instead in offline case. Recently, 
there have been a large amount of literature that leverages the duality in RL to develop a novel algorithm estimating quantities related to visitation distributions. In a field of offline IL, DemoDICE~\cite{kim2021demodice} proposed to solve it by deriving the dual of visitation distribution matching problem with an explicit regularizer minimizing the $f$-divergence between the visitation distribution induced by policy and the supplement dataset distribution. SMODICE~\cite{ma2022versatile} considers the IfO problem, which solves IL with state-only demonstrations by expert, and proposes a versatile offline IL algorithm by leveraging $f$-divergence and Fenchel duality instead of KL used in~\cite{kim2021demodice}.

\section{Experiments}
\label{experiment}
We evaluate the performance of our algorithm in both discrete and continuous MDP. We first show the trade-off effect by the discount factor in finite-discrete MDP, then evaluate the imitation performance of our algorithm using IGI in continuous MDP.
\subsection{Finite and Discrete MDP}
In finite-discrete MDP, we empirically show the trade-off effect by the discount factor in the Random MDP. This environment generates the finite and discrete MDP randomly. We follow the environment configuration of ~\cite{lee2021optidice}. Details of the experiment in finite-discrete MDP and its results with various settings are shown in Appendix \ref{discrete result} due to the lack of space.

\subsection{Continuous MDP}
For continuous MDP, we evaluate the imitation performance of our algorithm using IGI in the MuJoCo continuous control environment \cite{todorov2012mujoco}: \texttt{HalfCheetah-v2}, \texttt{Hopper-v2}, \texttt{Walker2d-v2} and \texttt{Ant-v2}. The dataset used for learning is obtained from the D4RL datasets \cite{fu2020d4rl}. We sampled expert trajectories and random trajectories from the \texttt{expert} and \texttt{random} dataset of the D4RL benchmark, respectively. We use a single expert trajectory for $\mathcal{D}^E$, and our suboptimal dataset $\mathcal{D}^O$ is constructed as a union of expert trajectories and random-policy trajectories with specific ratios.

\begin{figure}[t!]
    \centering
    \includegraphics[width=\textwidth]{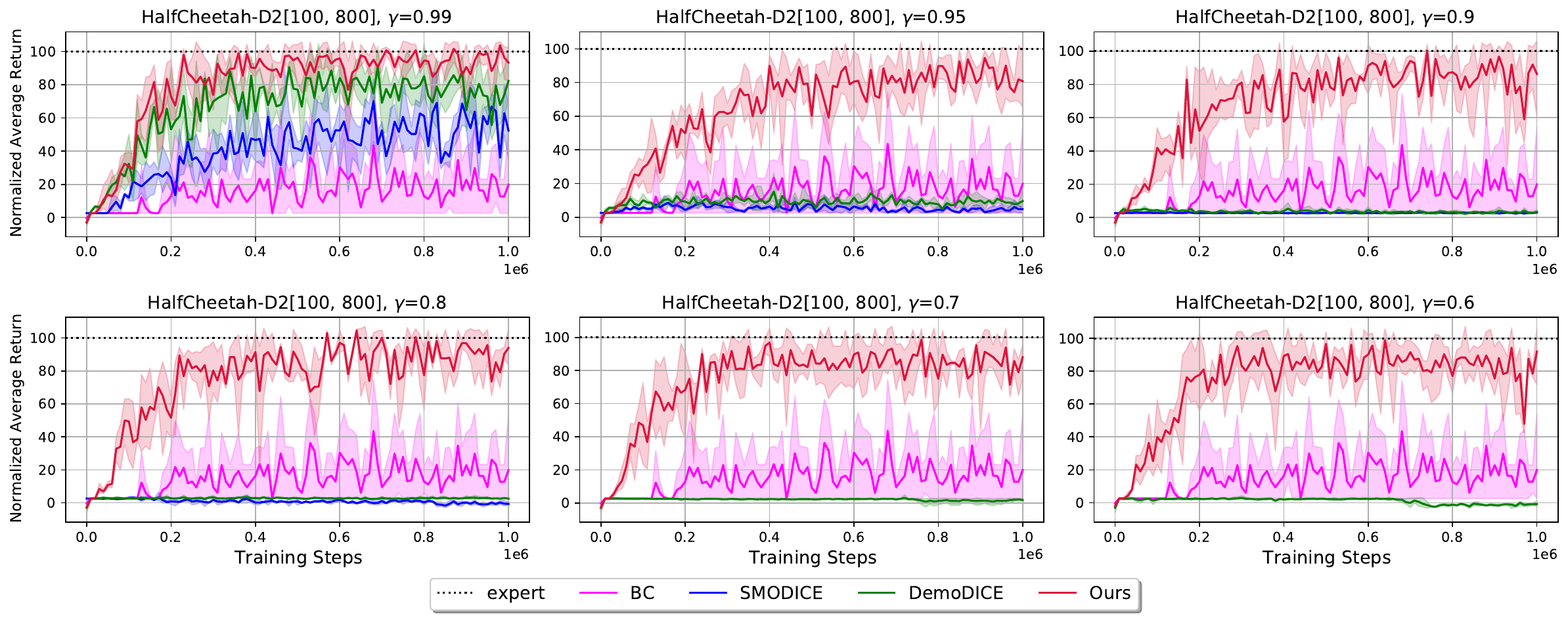}
    \caption{Performance of our algorithm (IGI) and baselines according to different $\gamma$s on \texttt{HalfCheetah-v2} with D2 dataset. Here, the shaded area shows the standard error of the normalized evaluation over 3 seeds.}
    \label{experiment_allgamma}
\end{figure}

We conduct experiments by changing the ratio between the expert data and the random-policy data of the suboptimal dataset. We call suboptimal dataset consisting of 400 expert trajectories and 800 random-policy trajectories as D1, 100:800 as D2, 50:800 as D3. We compare our algorithm using IGI with other offline IL algorithms, DemoDICE \cite{kim2021demodice}, SMODICE \cite{ma2022versatile}, and BC. Performance of other algorithms is measured based on the official hyperparameters made public by authors without any modification. For BC, we show the result of learning with $\mathcal{D}^D$, which has the best performance among the various combination of the data to be cloned. All experiment results in this paper are averaged over 3 seeds and normalized so that 0 corresponds to the average score of the random-policy dataset, and 100 corresponds to the average score of the expert policy dataset.

In order to demonstrate the robustness of our algorithm over different $\gamma$s, we show the results by applying various discount factor. We use discount factor for $\gamma\in\{0.99, 0.95, 0.9, 0.8, 0.7, 0.6\}$. We also tested every algorithm with $\gamma$ below 0.6, but due to exploding gradient, it was not possible to run DemoDICE and SMODICE for lower $\gamma$s below $0.6$. Therefore, we report only the result of $\gamma$s mentioned above.

The results of baselines on different $\gamma$s are shown in Figure \ref{experiment_allgamma}. It can be seen that the introduced method IGI makes the algorithm perform well regardless of $\gamma$. Additional results on other environments and D1, D3-ratio datasets are shown in the Appendix \ref{continuous result}. In addition, the experiment in Figure \ref{experiment_allgamma} has too large enough size of the dataset to make $\epsilon_P$ in the second term of Theorem \ref{theorem1} negligible, and the effect of controlling the discount factor is not clearly visible. Therefore, we make additional experiments by lowering the amount of the dataset about 10 times less (40 expert trajectories and 80 random-policy trajectories) than the D1 setting to clearly show the effect of controlling the discount factor. Relevant results are shown in Appendix \ref{4080}. Also, the pseudocode and the hyperparameters used to demonstrate IGI can be found in Appendix \ref{experiment details}.

\section{Conclusion}
\label{conclusion}
In this paper, we analyze the effect of controlling the discount factor on offline IL and motivate that the discount factor can take a role of a regularizer to prevent the sampling error of the supplementary dataset from hurting the performance. We show that the previously suggested imitation learning algorithms that utilize discriminators and a visitation distribution matching objective suffer from the discrepancy between the visitation distribution and the empirical distribution when a low discount factor $\gamma$ is applied. To this end, we proposed Inverse Geometric Initial state sampling (IGI), which uses the whole dataset with the weighting inversely proportional to the geometric distribution, to alleviate the problem that we cannot recover the expert policy $\pi^E$. We show that our approach shows stable and competitive performance regardless of the discount factor compared to other visitation distribution matching algorithms with explicit regularization.

\bibliographystyle{splncs04}
\bibliography{splncs04}

\appendix
\section{Proof of Theorem \ref{theorem1}}
\label{proof of theorem 1}
The proof of Theorem \ref{theorem1} is based on the analysis in previous works \cite{petrik2008biasing,jiang2015dependence,lee2020batch,xu2020error}. We provide three Lemmas first and use them to prove the Theorem \ref{theorem1}. To prove three Lemmas, we define several notations. $\textbf{P}$ and $\hat{\textbf{P}}$ denote a matrix of the underlying transition dynamics and estimated transition dynamics, respectively. $\mathbb{E}_{d^{\pi}_{P,\gamma}}[r(s,a)]$ denotes the expected total reward discounted by $\gamma$ for the policy $\pi$ under transition dynamics $\textbf{P}$. $\textbf{d}_0$ and $\textbf{r}$ denote a vector of the initial state probability and the reward, respectively. For $\textbf{r}$ and $\textbf{P}$, if the $\pi$ is at superscript, it means ``following policy $\pi$''.
\begin{lemma}
\label{gamma bound}
    For any MDP with bounded rewards $|r(s,a)|\leq R_{\text{max}}$, for all $\pi:S\rightarrow A$ and $\hat{\gamma}\leq\gamma$,
    \begin{equation*}
        \left|\mathbb{E}_{d^{\pi}_{P,\gamma}}\left[r(s,a)\right]-\mathbb{E}_{d^{\pi}_{P,\hat{\gamma}}}\left[r(s,a)\right]\right|\leq R_{\text{max}}\frac{\gamma-\hat{\gamma}}{(1-\gamma)(1-\hat{\gamma})}.
    \end{equation*}
\end{lemma}
\begin{proof}
    \begin{align*}
        \left|\mathbb{E}_{d^{\pi}_{P,\gamma}}\left[r(s,a)\right]-\mathbb{E}_{d^{\pi}_{P,\hat{\gamma}}}\left[r(s,a)\right]\right|&= \left\vert\mathbb{E}_{d^{\pi}_{P,\hat{\gamma}}}\left[r(s,a)\right]-\mathbb{E}_{d^{\pi}_{P,\gamma}}\left[r(s,a)\right] \right\vert\\
        &\leq\left\Vert \sum_{t=0}^\infty\left( \gamma^t-\hat{\gamma}^t\right){\textbf{P}^\pi}^t \textbf{r}^\pi\right\Vert_\infty\\
        &\leq \sum_{t=0}^\infty\left( \gamma^t-\hat{\gamma}^t\right)R_{\text{max}}\\
        &= \left(\frac{1}{1-\gamma}-\frac{1}{1-\hat{\gamma}} \right)R_{\text{max}}\\
        &= \frac{\gamma-\hat{\gamma}}{(1-\gamma)(1-\hat{\gamma})}R_{\text{max}}
    \end{align*}
\end{proof}
We define one more notation for Lemma \ref{mdp bound} and \ref{policy bound}.  $\textbf{d}^\pi_{P,\gamma}$ denotes a vector of the marginal state probability induced by $\pi$ under transition dynamics $\textbf{P}$ using $\gamma$. 
\begin{lemma}
\label{mdp bound}
    We can bound the difference of the evaluations of policy $\pi$ on two different MDPs with bounded rewards $|r(s,a)|\leq R_{\text{max}}$ as
    \begin{equation*}
    \left|\mathbb{E}_{d^{\pi}_{P,\gamma}}\left[r(s,a)\right]-\mathbb{E}_{d^{\pi}_{\widehat{P},\gamma}}\left[r(s,a)\right]\right|\leq \frac{2\gamma R_{\text{max}}}{(1-\gamma)^2}\mathbb{E}_{\begin{subarray}{l}s\sim d^\pi_{\widehat{P},\gamma},\\ a\sim\pi\end{subarray}}\left[\mathbb{TV}\left(\hat{P}\left(s^\prime\vert s,a\right)\Vert P\left(s^\prime\vert s,a\right)\right)\right].
\end{equation*}
\end{lemma}
\begin{proof}
    \begin{align}
        &\left|\mathbb{E}_{d^\pi_{P,\gamma}}[ r(s,a) ] - \mathbb{E}_{d^\pi_{\widehat{P},\gamma}}[ r(s,a) ]\right|=\left|\mathbb{E}_{d^\pi_{\widehat{P},\gamma}}[ r(s,a) ] - \mathbb{E}_{d^\pi_{P,\gamma}}[ r(s,a) ]\right|\\
        \label{using start}
        &=\left|\textbf{r}^\top\left[(I-\gamma\hat{\textbf{P}})^{-1}-(I-\gamma\textbf{P})^{-1}\right]\textbf{d}_0\right| \\
        &=\left| \textbf{r}^\top(I-\gamma\textbf{P})^{-1}\left[\gamma\hat{\textbf{P}}-\gamma\textbf{P} \right](I-\gamma\hat{\textbf{P}})^{-1}\textbf{d}_0\right|\\
        &=\gamma\left| \textbf{r}^\top(I-\gamma\textbf{P})^{-1}\left[\hat{\textbf{P}}-\textbf{P} \right]\frac{\textbf{d}^\pi_{\widehat{P}}}{1-\gamma}\right|\\
        &\leq\frac{\gamma}{1-\gamma}\left\Vert \textbf{r}^\top(I-\gamma\textbf{P})^{-1}\right\Vert_\infty \left\Vert\left[\hat{\textbf{P}}-\textbf{P} \right]\textbf{d}^\pi_{\widehat{P}} \right\Vert_1\\
        \label{using end}
        &\leq\frac{\gamma R_{\text{max}}}{(1-\gamma)^2}\left\Vert\left[\hat{\textbf{P}}-\textbf{P} \right]\textbf{d}^\pi_{\widehat{P}} \right\Vert_1\\
        &\leq\frac{\gamma R_{\text{max}}}{(1-\gamma)^2}\sum_{s^\prime,a,s}\left| \hat{P}(s^\prime\vert s,a)-P(s^\prime\vert s,a)\right|\pi(a\vert s)d^\pi_{\widehat{P}}(s)\\
        &=\frac{2\gamma R_{\text{max}}}{(1-\gamma)^2}\mathbb{E}_{\begin{subarray}{l}s\sim d^\pi_{\widehat{P},\gamma},\\ a\sim\pi\end{subarray}}\left[\mathbb{TV}\left(\hat{P}\left(s^\prime\vert s,a\right)\Vert P\left(s^\prime\vert s,a\right)\right)\right]
    \end{align}
\end{proof}
\begin{lemma}
\label{policy bound}
    We can bound the difference of the evaluations of two policies $\pi$ and $\mu$ with bounded rewards $|r(s,a)|\leq R_{\text{max}}$ as
    \begin{equation*}
        \left|\mathbb{E}_{d^{\pi}_{P,\gamma}}\left[r(s,a)\right]-\mathbb{E}_{d^{\mu}_{P,\gamma}}\left[r(s,a)\right]\right|\leq\frac{2R_{\text{max}}}{(1-\gamma)^2}\mathbb{E}_{s\sim d^\pi_P}\left[\mathbb{TV}(\pi(a\vert s)\Vert\mu(a\vert s)) \right].
    \end{equation*}
\end{lemma}
\begin{proof}
We can expand the difference of the evaluations of two policies like
    \begin{align*}
        |\mathbb{E}_{d^\pi_{P,\gamma}}[ r(s,a) ] - \mathbb{E}_{d^{\mu}_{P,\gamma}}[ r(s,a) ]|=\left| \left[(\textbf{r}^\pi)^\top(I-\gamma\textbf{P}^\pi)^{-1}-(\textbf{r}^{\mu})^\top(I-\gamma\textbf{P}^{\mu})^{-1} \right]\textbf{d}_0\right|&\\
        \leq\left| (\textbf{r}^\pi-\textbf{r}^{\mu})^\top\frac{\textbf{d}^\pi_{P}}{1-\gamma}\right|+\left|(\textbf{r}^{\mu})^\top\left[(I-\gamma\textbf{P}^\pi)^{-1}-(I-\gamma\textbf{P}^{\mu})^{-1} \right]\textbf{d}_0\right|&.
    \end{align*}
The first term is bounded as
    \begin{equation}
    \label{policy bound_1}
        \left| (\textbf{r}^\pi-\textbf{r}^{\mu})^\top\frac{\textbf{d}^\pi_{P}}{1-\gamma}\right|\leq\frac{2R_{\text{max}}}{1-\gamma}\mathbb{E}_{s\sim d^\pi_{P,\gamma}}\left[\mathbb{TV}(\pi(a\vert s)\Vert \mu(a\vert s)) \right].
    \end{equation}
Using the process (\ref{using start}-\ref{using end}), the second term is bounded as
    \begin{align}
        \left|(\textbf{r}^{\mu})^\top\left[(I-\gamma\textbf{P}^\pi)^{-1}-(I-\gamma\textbf{P}^{\mu})^{-1} \right]\textbf{d}_0\right|\leq\frac{\gamma R_{\text{max}}}{(1-\gamma)^2}\left\Vert \left[\textbf{P}^\pi-\textbf{P}^{\mu}\right]\textbf{d}^\pi_{P}\right\Vert_1\\
        =\frac{\gamma R_{\text{max}}}{(1-\gamma)^2}\sum_{s^\prime}\left| \sum_{s,a}P(s^\prime\vert s,a)\left[\pi(a\vert s)-\mu(a\vert s)\right]d^\pi_{P}(s)\right|\\
        \leq \frac{\gamma R_{\text{max}}}{(1-\gamma)^2}\sum_{s,a}\left| \pi(a\vert s)-\mu(a\vert s)\right|d^\pi_{P}(s)\\
        \label{policy bound_2}
        =\frac{2\gamma R_{\text{max}}}{(1-\gamma)^2}\mathbb{E}_{s\sim d^\pi_{P,\gamma}}\left[\mathbb{TV}\left(\pi(a\vert s)\Vert\mu(a\vert s)\right) \right].
    \end{align}
Combining (\ref{policy bound_1}) and (\ref{policy bound_2}), we have the proof of Lemma \ref{policy bound}.
\end{proof}
\bound*

\begin{proof}
The error bound for imitated policy $\pi$ can be expanded as
\begin{align*}
    &\left| \mathbb{E}_{d^\pi_{P,\gamma}}[ r(s,a) ]-\mathbb{E}_{d^{\pi^E}_{P,\gamma}}[ r(s,a) ] \right| \\
    &\le \left|\mathbb{E}_{d^\pi_{P,\gamma}}[ r(s,a) ] - \mathbb{E}_{d^\pi_{P,\widehat{\gamma}}}[ r(s,a) ]\right| + \left|\mathbb{E}_{d^\pi_{P,\widehat{\gamma}}}[ r(s,a) ] - \mathbb{E}_{d^\pi_{\widehat{P},\widehat{\gamma}}}[ r(s,a) ]\right|\\
    &+\left|\mathbb{E}_{d^\pi_{\widehat{P},\widehat{\gamma}}}[ r(s,a) ] - \mathbb{E}_{d^{\pi^E}_{\widehat{P},\widehat{\gamma}}}[ r(s,a) ]\right| + \left|\mathbb{E}_{d^{\pi^E}_{P,\widehat{\gamma}}}[ r(s,a) ] - \mathbb{E}_{d^{\pi^E}_{\widehat{P},\widehat{\gamma}}}[ r(s,a) ]\right|\\
    &+ \left|\mathbb{E}_{d^{\pi^E}_{P,\widehat{\gamma}}}[ r(s,a) ] - \mathbb{E}_{d^{\pi^E}_{P,\gamma}}[ r(s,a) ]\right|.
\end{align*}
The first and last terms are bounded by Lemma \ref{gamma bound}, and the second and fourth terms are bounded by Lemma \ref{mdp bound}. Lastly, the third term is bounded by Lemma \ref{policy bound}. This completes the proof.
\end{proof}

\section{Proof for simple MDP}
\label{proof for simple MDP}
In Section \ref{4.1}, we define expert policy as 
\[\pi^E(a_1\vert s_0)=0.6,\,\pi^E(a_2\vert s_0)=0.4,\,\pi^E(a_1\vert s_1)=1\,\,\pi^E(a_1\vert g)=1.\]
we said if N trajectories are sampled, the number of $\tau_1$ and $\tau_2$ is 0.6N and 0.4N. Since each trajectories have three $(s,a)$ pairs, total number of $(s,a)$ pairs is 3N. Calculating the previously defined $E(s,a)$ is as follows
\begin{align*}
    &E(s_0,a_1)=\frac{0.6N}{3N} = \frac{1}{5},\;E(s_0,a_2)=\frac{0.4N}{3N} = \frac{2}{15}\\&E(s_1,a_1)=\frac{0.6N}{3N} = \frac{1}{5},\;E(g,a_1)=\frac{1.4N}{3N} = \frac{7}{15}
\end{align*}
Next, imitation policy is defined as
\[\pi(a_1\vert s_0)=\theta,\,\pi(a_2\vert s_0)=1-\theta,\,\pi(a_1\vert s_1)=1\,\,\pi(a_1\vert g)=1.\]
Then, $d^\pi(s,a)$ has the following values
\begin{align*}
    d^\pi(s_0,a_1)&=\pi(a_1|s_0)\left[(1-\gamma)p_0(s_0)+\gamma\sum_{\Bar{s},\Bar{a}}P(s_0\vert\Bar{s},\Bar{a})d^\pi(\Bar{s},\Bar{a})\right] = \theta(1-\gamma)\\
    d^\pi(s_0,a_2)&=\pi(a_2|s_0)\left[(1-\gamma)p_0(s_0)+\gamma\sum_{\Bar{s},\Bar{a}}P(s_0\vert\Bar{s},\Bar{a})d^\pi(\Bar{s},\Bar{a})\right] = (1-\theta)(1-\gamma)\\
    d^\pi(s_1,a_1)&=\pi(a_1|s_1)\left[(1-\gamma)p_0(s_1)+\gamma\sum_{\Bar{s},\Bar{a}}P(s_1\vert\Bar{s},\Bar{a})d^\pi(\Bar{s},\Bar{a})\right]\\
    &= 1\cdot[(1-\gamma)\cdot0+\gamma\cdot1\cdot d^\pi(s_0,a_1)]=\theta\gamma(1-\gamma)\\
    d^\pi(g,a_1)&=\pi(a_1|g)\left[(1-\gamma)p_0(g)+\gamma\sum_{\Bar{s},\Bar{a}}P(g\vert\Bar{s},\Bar{a})d^\pi(\Bar{s},\Bar{a})\right]\\ &= 1\cdot[0+\gamma(1\cdot d^\pi(s_0,a_2)+1\cdot d^\pi(s_1,a_1)+1\cdot d^\pi(g,a_1))]\\
    &=(1-\theta)\gamma(1-\gamma)+\theta\gamma^2(1-\gamma)+\gamma d^\pi(g,a_1)\\
    &\rightarrow\,d^\pi(g,a_1) = (1-\theta)\gamma+\theta\gamma^2
\end{align*}
Using E(s,a) and $d^\pi(s,a)$, $D_{KL}(d^\pi\Vert E)$ is expressed as 
\begin{align*}
    \min_\pi\; &D_{KL}(d^\pi(s,a)\Vert E(s,a))\\&=d^\pi(s_0,a_1)\log\frac{d^\pi(s_0,a_1)}{E(s_0,a_1)}+d^\pi(s_0,a_2)\log\frac{d^\pi(s_0,a_2)}{E(s_0,a_2)}\\&+d^\pi(s_1,a_1)\log\frac{d^\pi(s_1,a_1)}{E(s_1,a_1)}+d^\pi(g,a_1)\log\frac{d^\pi(g,a_1)}{E(g,a_1)}
\end{align*}
\begin{align*}
     \min_\theta\; &D_{KL}(d^\pi(s,a)\Vert E(s,a))\\&=\theta(1-\gamma)\log\frac{\theta(1-\gamma)}{1/5}+(1-\theta)(1-\gamma)\log\frac{(1-\theta)(1-\gamma)}{2/15}\\&+\theta\gamma(1-\gamma)\log\frac{\theta\gamma(1-\gamma)}{1/5}+\left((1-\theta)\gamma+\theta\gamma^2\right)\log\frac{(1-\theta)\gamma+\theta\gamma^2}{7/15}
\end{align*}
To find a minimizer $\theta^*$ of above optimization, we use first-order optimality condition. That is,
\[\frac{\partial D_{KL}(d^\pi(s,a)\Vert E(s,a))}{\partial\theta} = 0.\]
We derive first derivative of $D_{KL}(d^\pi(s,a)\Vert E(s,a))$. Result is
\begin{align*}
\begin{split}
    \frac{\partial D_{KL}(d^\pi(s,a)\Vert E(s,a))}{\partial\theta} &=(1-\gamma)\left[(1+\gamma)\log\theta-\log(1-\theta)+\log\frac{2}{3}\right]\\ &+\gamma(1-\gamma)\left[\log\frac{7\gamma(1-\gamma)}{3}-\log((1-\theta)\gamma+\theta\gamma^2)\right]=0
\end{split}
\end{align*}
It is hard to get the exact $\theta^*$ value because equation is very complicated. Therefore, we substitute the probability value $\frac{3}{5}$, which is probability of expert policy, into $\theta$ to check whether the derivative becomes 0.
\begin{align*}
\begin{split}
    &\frac{\partial D_{KL}(d^\pi(s,a)\Vert E(s,a))}{\partial\theta}\\& = (1-\gamma)\left[(1+\gamma)\log\frac{3}{5}-\log(1-\frac{3}{5})+\log\frac{2}{3}\right]\\&+\gamma(1+\gamma)\left[\log\frac{7\gamma(1-\gamma)}{3}-\log((1-\frac{3}{5})\gamma+\frac{3}{5}\gamma^2)\right]=0\\
    &\rightarrow \gamma\log\frac{3}{5}+\gamma\left[\log\frac{7\gamma(1-\gamma)}{3}-\log\left(\frac{2}{5}\gamma+\frac{3}{5}\gamma^2\right)\right]=0\\
    &\rightarrow \log\frac{3}{5}+\log\frac{7\gamma(1-\gamma)}{3}-\log\left(\frac{2}{5}\gamma+\frac{3}{5}\gamma^2\right)=\log\left(\frac{7(1-\gamma)}{2+3\gamma}\right)=0\\
    &\rightarrow 7(1-\gamma)=2+3\gamma
\end{split}
\end{align*}
Since $\gamma\in[0,1)$, we exclude $\gamma=1$ case. Solving last equation yields $\gamma=0.5$. In summary, if $\gamma=0.5$, the expert policy can be recovered through optimization $\min_\pi D_{KL}(d^\pi(s,a)\Vert E(s,a))$, but if other $\gamma$ values are applied, it can be interpreted that the optimal policy is not the same as the expert policy. 

\section{Finding IGI Distribution}
\label{IGI}
Our goal is to find the initial distribution $\tilde{p}_0$ that makes the distribution made by $t_0+t_{geom}$ uniform, where $t_0\sim\tilde{p}_0(t)$ and $t_{geom}\sim \text{Geom}_{1-\gamma}(t\vert t_0)$. Note that $\text{Geom}_{1-\gamma}(t\vert t_0)$ is conditional distribution. Because, in general, there is a maximum timestep, $|T|$, for episode in the learning environment, so in order to prevent the sum of timestep exceeds over $|T|$, $\text{Geom}_{1-\gamma}(t\vert t_0)$ must be a conditional distribution that considers the sampled initial timestep.
Let $T=t_0+t_{geom}$, Then $P_T(T)$ is defined as
\begin{align}
\begin{split}
\label{C.1}
    P_T(T) &= \sum^T_{t_0=0}\text{Geom}_{1-\gamma}(T-t_0|t_0)\tilde{p}_0(t_0)\\
    &= \sum^T_{t_0=0}\frac{(1-\gamma)\gamma^{T-t_0}}{\sum_{i=0}^{|T|-t_0}(1-\gamma)\gamma^i}\tilde{p}_0(t_0)\\
    &= \sum^T_{t_0=0}\frac{(1-\gamma)\gamma^{T-t_0}}{Sum(|T|-t_0)}\tilde{p}_0(t_0).
\end{split}
\end{align}
For the convenience of notation, we rewrite $\sum_{i=0}^{|T|-t_0}(1-\gamma)\gamma^i$ as $Sum(|T|-t_0)$. We can express Equation (\ref{C.1}) in linear problem form. To this end, it is represented as follows.
\[\underbrace{
\begin{bmatrix}
P_T(0)\\\\\\
\vdots\\\\\\
P_T(|T|)
\end{bmatrix}}_{P_T}
=\underbrace{
\begin{bmatrix}
\frac{1-\gamma}{Sum(|T|)} & 0 & 0 & \cdots & 0 \\
\frac{(1-\gamma)\gamma}{Sum(|T|)} & \frac{1-\gamma}{Sum(|T|-1)} & 0 & \cdots & 0 \\
& \frac{(1-\gamma)\gamma}{Sum(|T|-1)} & \frac{1-\gamma}{Sum(|T|-2)} & \cdots & 0 \\
\vdots & \vdots & \vdots & \ddots & \vdots \\
& &  & \\
\frac{(1-\gamma)\gamma^{|T|}}{Sum(|T|)} & \frac{(1-\gamma)\gamma^{|T|-1}}{Sum(|T|-1)} & \frac{(1-\gamma)\gamma^{|T|-2}}{Sum(|T|-2)} & \cdots & 1-\gamma\;
\end{bmatrix}}_{P_{t_{geom}}}
\underbrace{\begin{bmatrix}
\tilde{p}_0(0)\\\\\\
\vdots\\\\\\
\tilde{p}_0(|T|)
\end{bmatrix}}_{\tilde{p}_0}
\]
In our implementation, we set $P_T(T) = \frac{\textrm{The number of data with timestep T in dataset}}{\textrm{Total number of data in dataset}}$. Finally, solving $\tilde{p}_0 = P^{-1}_{t_{geom}}\cdot P_T$ gives us the $\tilde{p}_0(t)$ we want.
\section{Experimental Details}
\label{experiment details}
\subsection{Algorithm details}
In this subsection, we provide a simple pseudocode for our algorithm using IGI in Algorithm \ref{algorithm 1}.
\begin{algorithm}
\begin{algorithmic}
\caption{Our algorithm using IGI}
\Require dataset $\mathcal{D}^D=\mathcal{D}^E\cup \mathcal{D}^O$, IGI distribution $\tilde{p}_0$ from Appendix \ref{IGI}, policy network $\pi_\theta$ with parameter $\theta$, neural networks $\nu_\phi$ and $c_\psi$ with parameter $\phi$ and $\psi$, discount factor $\gamma$, batch size $B$, learning rate $\alpha$
\Ensure $P_{T_\mathcal{D}}$ is a distribution $P_T$ from Appendix \ref{IGI} using dataset $\mathcal{D}$, and $\mathcal{D}_t$ is an dataset $\mathcal{D}$ consisted of data with timestep $t$\\

\State Make IGI distribution $\tilde{p}_0$ using $\gamma$ and $P_{T_{\mathcal{D}^D}}$

\While{total iterations}
    \State Sample initial timestep: $\{t_0^{(i)}\}^{B}_{i=1}\sim\tilde{p}_0$\Comment{using IGI distribution}
    \State Sample uniform timestep: $\{t^{(i)}\}^{B}_{i=1}\sim P_{T_{\mathcal{D}^D}}$
    \State Sample expert timestep: $\{t_e^{(i)}\}^{B}_{i=1}\sim P_{T_{\mathcal{D}^E}}$\\
    
    \State Sample initial state: $\{s_0^{(i)}\}\sim \mathcal{D}^D_{t_0^{(i)}}$, for $i=1,...,B$
    \State Sample total data: $\{(s^{(i)}, a^{(i)}, {s^\prime}^{(i)})\}\sim \mathcal{D}^D_{t^{(i)}}$, for $i=1,...,B$
    \State Sample expert data: $\{(s_e^{(i)}, a_e^{(i)}, {s_e^\prime}^{(i)})\}\sim \mathcal{D}^E_{t_e^{(i)}}$, for $i=1,...B$\\
    
    \State Compute discriminator loss $J_c$ (\ref{3.2_11}) using $\{(s^{(i)},a^{(i)})\}^{B}_{i=1}$ and $\{(s^{(i)}_e,a^{(i)}_e)\}^{B}_{i=1}$
    \State Compute critic loss $J_\nu$ (\ref{3.2_9}) using $\{s^{(i)}_0\}^{B}_{i=1}$ and $\{(s^{(i)}, a^{(i)}, {s^\prime}^{(i)})\}^{B}_{i=1}$
    \State Compute policy loss $J_\pi$ (\ref{3.2_10}) using $\{(s^{(i)}, a^{(i)}, {s^\prime}^{(i)})\}^{B}_{i=1}$\\
    
    \State Update $\psi\leftarrow\psi-\alpha\nabla_\psi J_c$
    \State Update $\phi \leftarrow \phi-\alpha\nabla_\phi J_\nu$
    \State Update $\theta \leftarrow \theta-\alpha\nabla_\theta J_\pi$
\EndWhile
\label{algorithm 1}
\end{algorithmic}
\end{algorithm}
To compute three loss functions, we need sampled data from $d^E, D,$ and $p_0$. This begins by sampling the timesteps for each distribution. First of all, compute the IGI distribution using the geometric distribution $\text{Geom}(1-\gamma$) given a discount factor $\gamma$ as we described in Appendix \ref{IGI}. At this time, uniform distribution $P_T$ for the total dataset $D^D$ is also used for practical performance as shown in Figure \ref{irucompare} of Section \ref{4.2}. 

$P_T$ is defined as $\frac{\textrm{The number of data with timestep T in dataset}}{\textrm{Total number of data in dataset}}$ as in Appendix \ref{IGI}. Since IGI enables sampling for the discounted distribution uniformly regardless of the timestep, sampling from $d^E$ is replaced with sampling from $P_{T_{D^E}}$. After three types of timesteps are sampled, we need actual data corresponding to each type of timestep to use in loss functions. We uniformly sample actual data of corresponding timesteps in each three dataset. To compute the loss function (\ref{3.2_11}) of the discriminator, uniformly sampled expert data is used to compute the first term and the total data is used to compute the second term. In the case of the critic network, the initial data from IGI is used to compute the first term in (\ref{3.2_9}) and the total data is used to compute the second term. When we compute $A_\nu(s,a)$ in the second term, use $\gamma\nu({s^\prime}^{(i)})$ instead of $\gamma\mathcal{P}\nu(s^{(i)})$ and use the discriminator to compute $\log{\frac{d^E(s,a)}{D(s,a)}}$ as notified in the sentence below (\ref{3.2_11}). For the policy network, total data is used to compute (\ref{3.2_10}). Lastly, update the parameters of each network with a learning rate $\alpha$. Repeating this process for the total number of iterations.
\subsection{Implementation detail}
In this subsection, we provide the hyperparameter settings of our algorithm in Table \ref{hyperparameter table}. We use an absorbing state for practical implementation~\cite{kostrikov2018discriminator}. For discriminator learning, we use WGAN-GP~\cite{gulrajani2017improved} to achieve more robust learning. 
\begin{table}[htp!]
\caption{Hyperparameter settings}
\centering
\renewcommand{\arraystretch}{1.5}
\begin{tabular}{c|c}
\hline
\textbf{Hyperparameters}                   & \textbf{Our setting}                             \\ \hline
Actor learning rate                        & $3\times10^{-4}$ \\
Actor network size                         & $[256, 256]$                                   \\ \hline
Critic learning rate                       & $3\times10^{-4}$ \\
Critic network size                        & $[256, 256]$                                   \\
Critic gradient L2-norm coefﬁcient         & $10^{-4}$                       \\ \hline
Discriminator learning rate                & $3\times10^{-4}$ \\
Discriminator network size                 & $[256, 256]$                                  \\
Discriminator gradient  penalty coefﬁcient & 10                                               \\ \hline
Batch size                                 & 256                                              \\ \hline
Number of total iteration                  & $10^6$                       \\ \hline
Random seeds                               & 1, 2, 3                      \\ \hline
Kernel initializer                         & He normal initializer   \\ \hline
\end{tabular}
\label{hyperparameter table}
\end{table}

\newpage
\section{Experiment Results}
\subsection{Finite and Discrete MDP}
\label{discrete result}
In this subsection, we provide experiment details in finite-discrete MDP and its results.
\begin{figure}[H]
    \centering
    \includegraphics[width=\textwidth]{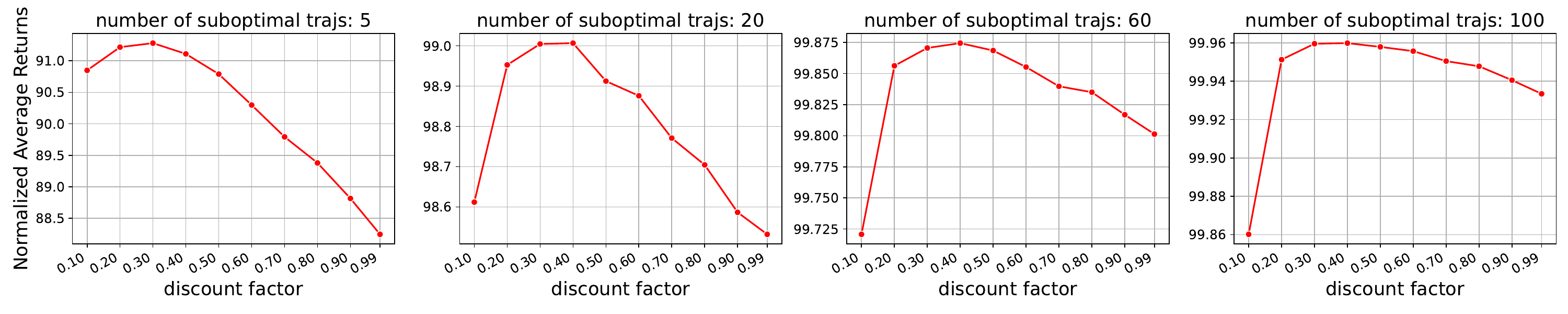}
    \caption{The trend of evaluation result with various discount factors in different settings. We plot the normalized average evaluation over 1000 random seeds.}
    \label{allgamma with different task}
\end{figure}
In finite and discrete MDP, we empirically show the trade-off effect by the discount factor in the Random MDP. This environment generates the finite and discrete MDP randomly. We follow the environment configuration of ~\cite{lee2021optidice}.

We construct the expert dataset $\mathcal{D}^E$ and suboptimal dataset $\mathcal{D}^O$ by rolling expert policy and suboptimal policy, respectively. We characterized the expert policy as the stochastic policy based on the optimal $Q$ (state-action) value of the randomly generated MDP. The suboptimal policy has performance between optimal and uniformly random policy by controlling the hyperparameter $\omega\in[0,1]$ as $\omega V^*(s_0)+(1-\omega)V^{\pi_{\text{unif}}}(s_0)$ where $V^*$ and $V^{\pi_{\text{unif}}}$ is the value function of optimal policy and uniformly random policy, respectively. Since the MDP is finite and discrete, we can compute the discounted visitation distribution of the policy directly, which means learning the discriminator is unnecessary. Thus, there is no need to use IGI in this environment. We use Maximum likelihood estimation (MLE) transition dynamics $\hat{P}$ and discount factor $\hat{\gamma}$ for training the policy as we discussed in Theorem \ref{theorem1}. We use discount factor for $\hat{\gamma}\in\{0.99, 0.9, 0.8, 0.7, 0.6, 0.5, 0.4, 0.3, 0.2, 0.1\}$ with varying the number of suboptimal trajectories for $\{5, 20, 60, 100\}$. The number of expert trajectory is 1 for all settings. Figure \ref{allgamma with different task} shows the policy evaluation result on the true MDP with true discount factor $\gamma$. It shows evaluation trends so that the trade-off effect by discount factor suggested in Section \ref{trade off by discount factor} can be confirmed. We can see that the performance can be optimized by choosing lower $\gamma$. Furthermore, the figure shows that the performance of using large discount factors becomes better as the number of suboptimal trajectories gets larger.
\subsection{Continuous MDP}
\label{continuous result}
We plot the learning curve according to $\gamma$ and dataset ratio. Figure \ref{D.1} represents the result in \texttt{HalfCheetah-v2} environment, and Figure \ref{D.2}, \ref{D.3}, \ref{D.4} is the result on the \texttt{Hopper-v2}, \texttt{Walker2d-v2} and \texttt{Ant-v2}, respectively.
\begin{figure}
    \centering
    \includegraphics[width=\textwidth]{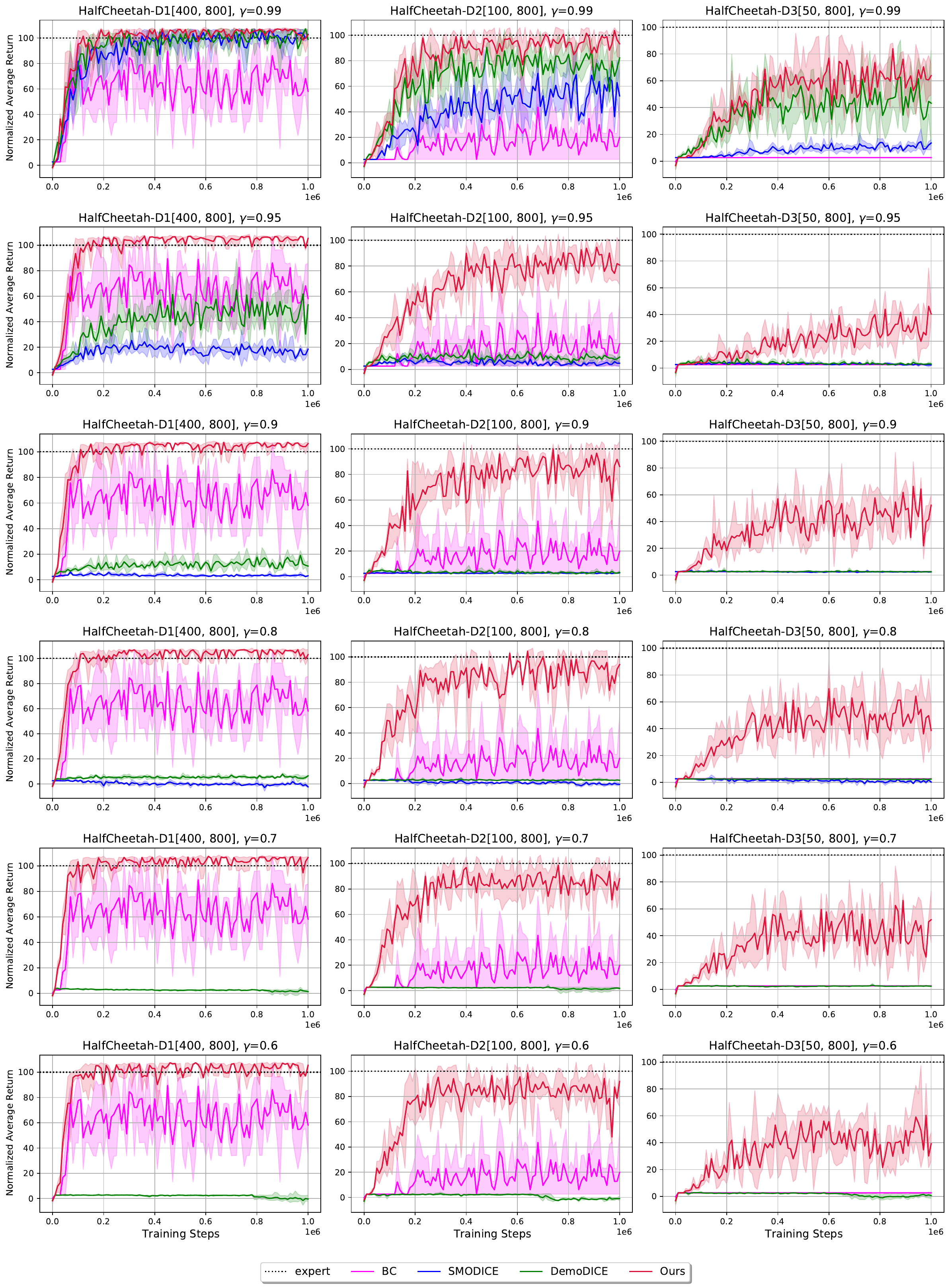}
    \caption{Performance of our algorithm using IGI and other baseline algorithms with $\gamma\in\{0.99, 0.95, 0.9, 0.8, 0.7, 0.6\}$ and dataset ratio D1, D2 and D3 in \texttt{HalfCheetah-v2} environment.}
    \label{D.1}
\end{figure}
\newpage
\begin{figure}[H]
    \centering
    \includegraphics[width=\textwidth]{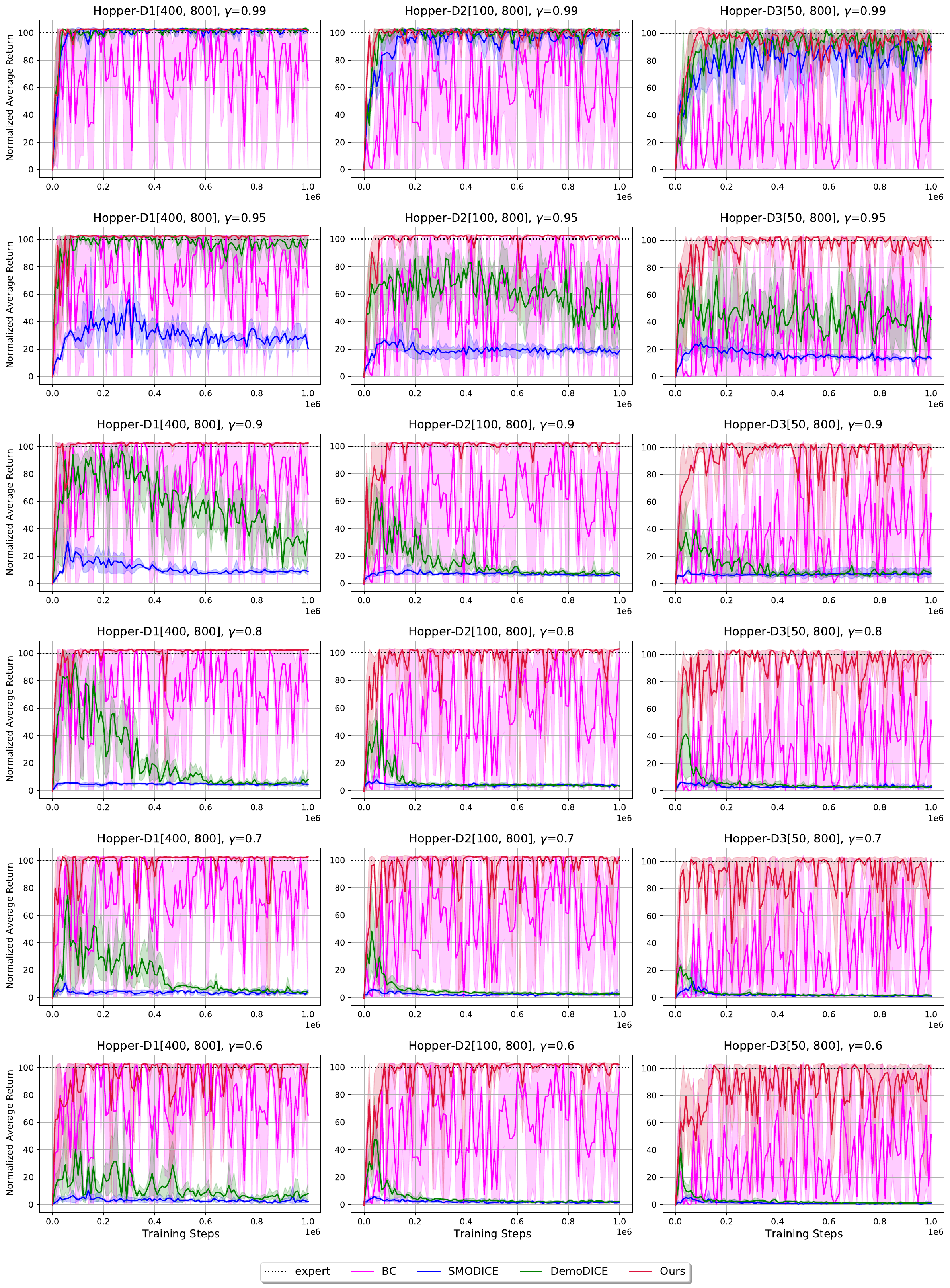}
    \caption{Performance of our algorithm using IGI and other baseline algorithms with $\gamma\in\{0.99, 0.95, 0.9, 0.8, 0.7, 0.6\}$ and dataset ratio D1, D2 and D3 in \texttt{Hopper-v2} environment.}
    \label{D.2}
\end{figure}
\newpage
\begin{figure}[H]
    \centering
    \includegraphics[width=\textwidth]{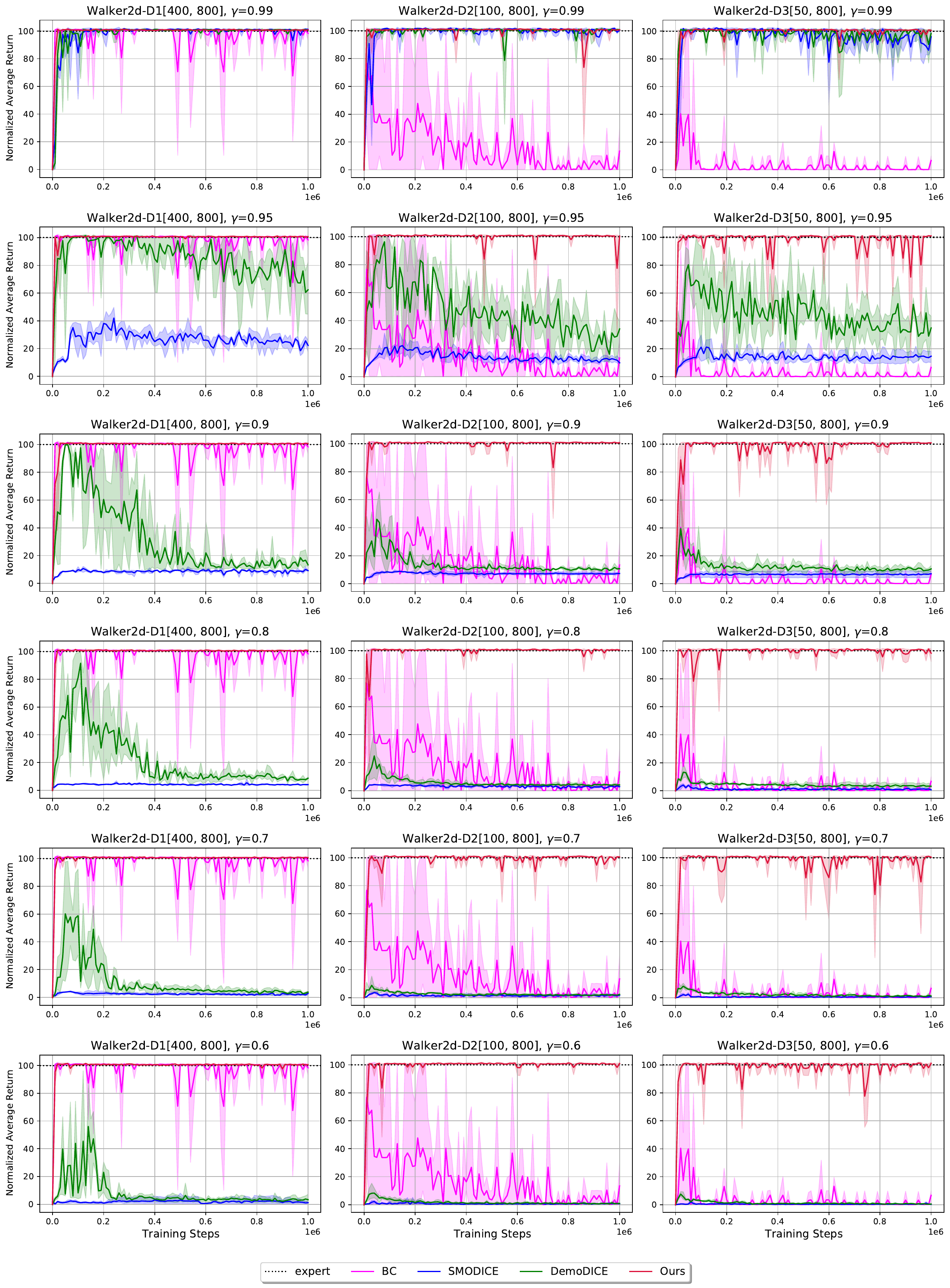}
    \caption{Performance of our algorithm using IGI and other baseline algorithms with $\gamma\in\{0.99, 0.95, 0.9, 0.8, 0.7, 0.6\}$ and dataset ratio D1, D2 and D3 in \texttt{Walker2d-v2} environment.}
    \label{D.3}
\end{figure}
\newpage
\begin{figure}[H]
    \centering
    \includegraphics[width=\textwidth]{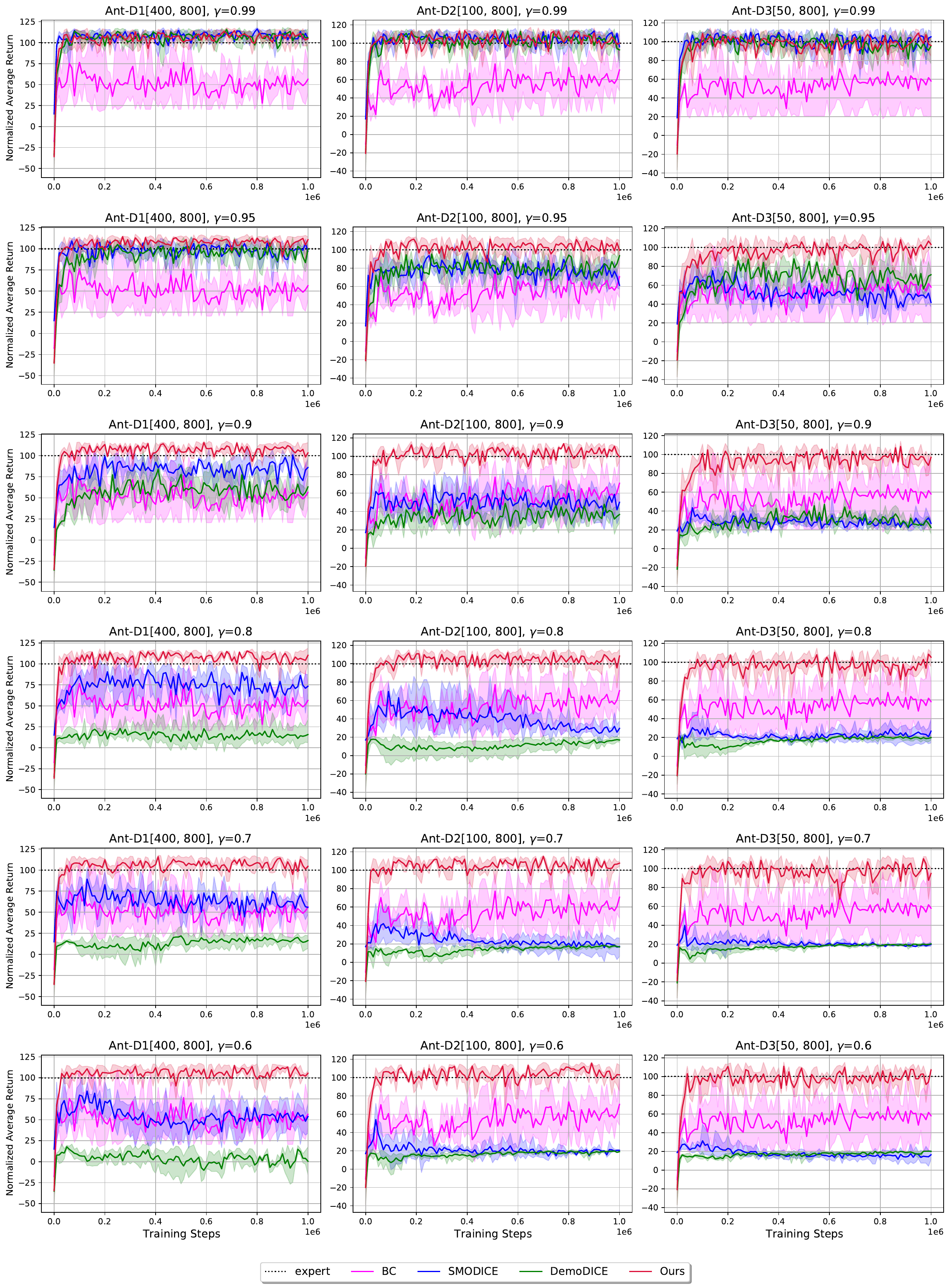}
    \caption{Performance of our algorithm using IGI and other baseline algorithms with $\gamma\in\{0.99, 0.95, 0.9, 0.8, 0.7, 0.6\}$ and dataset ratio D1, D2 and D3 in \texttt{Ant-v2} environment.}
    \label{D.4}
\end{figure}
\newpage
\section{Addtional Experiment Result}
\label{additional experiment}
\subsection{Experiment on Small Dataset in Continuous MDP}
\label{4080}
\begin{figure}[H]
    \centering
    \includegraphics[width=\textwidth]{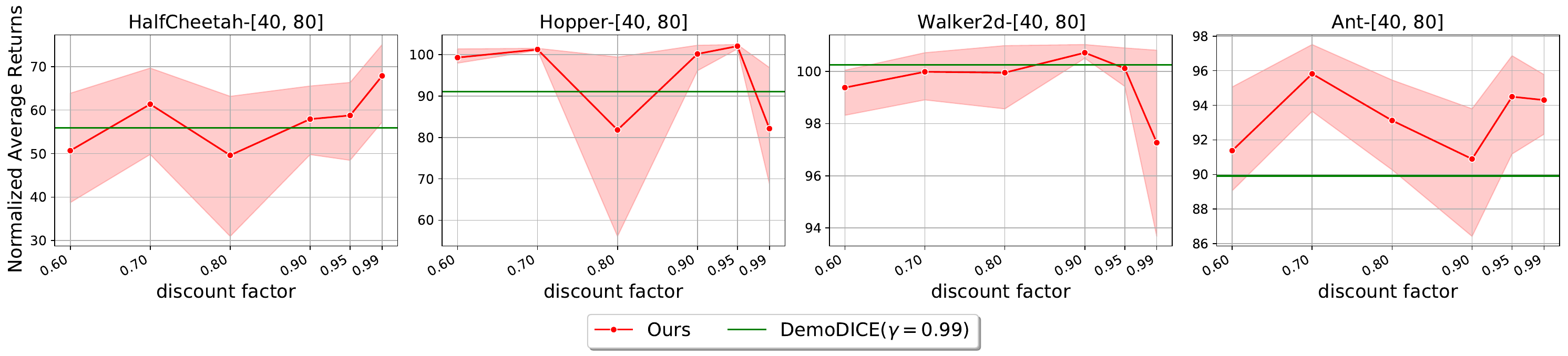}
    \caption{The trend of the final evaluation over 3 random seeds with various discount factors in different environment. For comparison, we plot green line which represents the normalized average returns of DemoDICE ($\gamma=0.99$).}
    \label{allgamma with 4080}
\end{figure}
We make additional experiments by lowering the amount of dataset about 10 times less (40 expert trajectories and 80 random-policy trajectories) than D1 / D2 / D3 settings. Figure \ref{allgamma with 4080} shows the last evaluation of IGI with applying aforementioned setting in 4 MuJoCo environments over 3 random seeds, and green line shown in the Figure \ref{allgamma with 4080} is the maximum last evaluation over 3 seed of DemoDICE among different discount factors for each environment (Note that best discount factor of DemoDICE is 0.99 for all environments).

As shown in the attached figure, we confirmed that the optimal discount factor appeared at lower than 0.99, except for the \texttt{HalfCheetah-v2} environment. To explain the reason why a different trend came out only in \texttt{HalfCheetah-v2}, we inform that unlike the other environments, \texttt{HalfCheetah-v2} environment has the peculiarity that learning is stable in the offline RL problem, even when the policy is unregularized. For example, refer to the experiment result by varying the policy regularization factor in Figure 9 of \cite{DBLP:journals/corr/abs-1911-11361}, \texttt{HalfCheetah-v2} environment shows good performance compared to other environments even when a low regularization factor $\alpha$ is applied. Combining these experimental results with the regularization effect on the discount factor in the offline RL referred to \cite{hu2022role} (Section 3), we can expect that the performance of the learned policy can be improved when a high discount factor is used in the \texttt{HalfCheetah-v2} environment. Eventually, the reason why Figure \ref{allgamma with 4080} shows a different trend only in the \texttt{HalfCheetah-v2} environment is interpreted as reflecting this environmental disposition, and the error bound that we proposed in Theorem \ref{theorem1} is valid under general situations.
\newpage
\subsection{Test on validity of Theorem \ref{theorem1}}
\label{F.2}
\begin{figure}[H]
    \centering
    \includegraphics[width=\textwidth]{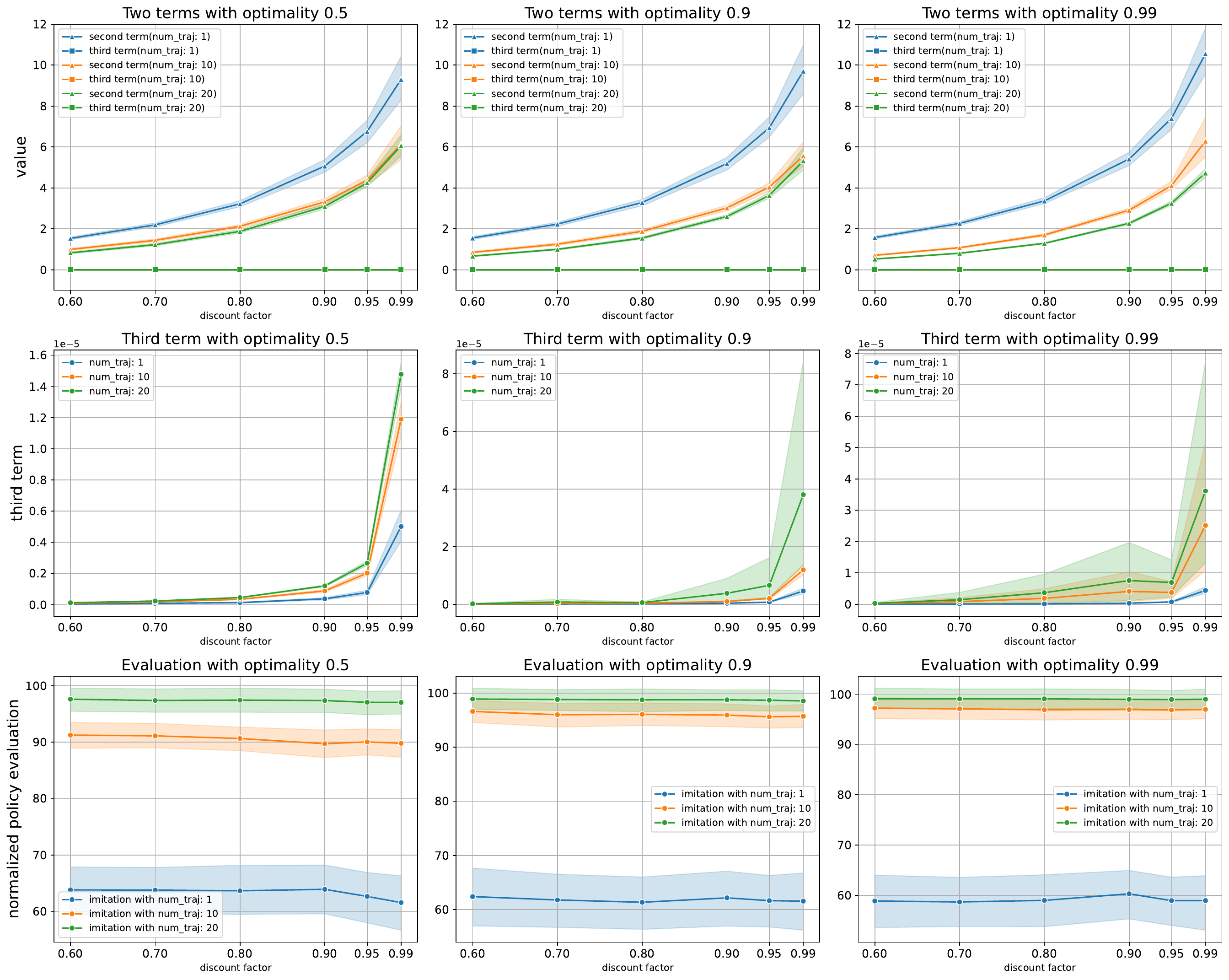}
    \caption{Figures in the first row show the magnitude of the second and third terms averaged over 100 random seeds with different discount factors and the number of suboptimal trajectories (num\_traj) for each optimality setting. The second row shows only the magnitude of the third term, and the last row represents the performance of imitated policies.}
    \label{two terms}
\end{figure}

To check the validity of Theorem \ref{theorem1}, we measured the magnitude of the second and third terms in finite-discrete MDP described in \ref{discrete result} where $\epsilon_P$ and $\epsilon_\pi$ can be exactly computed. Figure \ref{two terms} shows the trend of two terms in (\ref{theorem1_1}) with different discount factors at each setting. Results show that the third term has a much lower value than the second term. As a result, Theorem \ref{theorem1} is mainly dependent on the first and second terms.

The peculiar part is that the third term increases as the suboptimal data increases. In this situation, the proportion of expert data in the total dataset is reduced, and in that case, the state-marginal distribution, $d_{\widehat{P}, \hat{\gamma}}^\pi(s)$ generated by MLE MDP will give unnecessary amount of probability for the suboptimal data. Therefore, it is interpreted that the third term is increased because samples with high $D_{TV}(\pi\Vert\pi^E)$ are more frequently reflected in $\epsilon_\pi$ as suboptimal data increases.
\subsection{Performance on Different Suboptimal Dataset}
\begin{figure}[H]
    \centering
    \includegraphics[width=\textwidth]{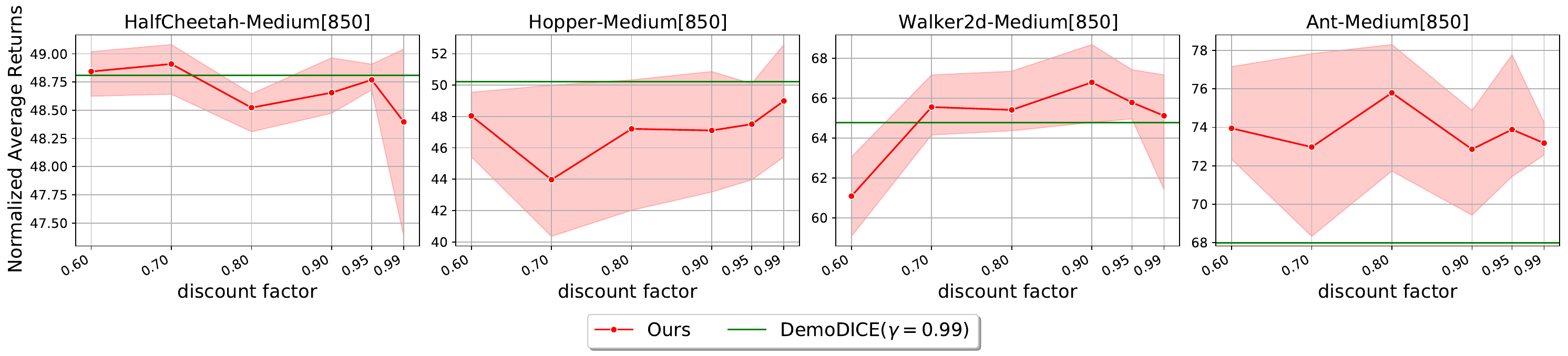}
    \caption{Last evaluation of IGI and best among DemoDICE when suboptimal dataset $\mathcal{D}^O$ is consisted of 850 medium-level trajectories. The gray line represents the best of DemoDICE 
 ($\gamma=0.99$).}
    \label{medium_last}
\end{figure}
We additionally conducted an experiment using medium-level trajectory in the D4RL dataset as suboptimal data. Figure \ref{medium_last} shows the evaluations of learning 1 million steps according to the discount factor when the suboptimal dataset is consisted of 850 medium-level trajectories, and the gray line represents DemoDICE's best performance. As a result, in an environment other than \texttt{Hopper-v2} it was possible to obtain a higher return than DemoDICE. We also included the learning curve of this experiment in the next page, Figure \ref{medium}. Here, we can see that the IGI still shows robust learning against the discount factor.
\newpage
\begin{figure}[H]
    \centering
    \includegraphics[width=\textwidth]{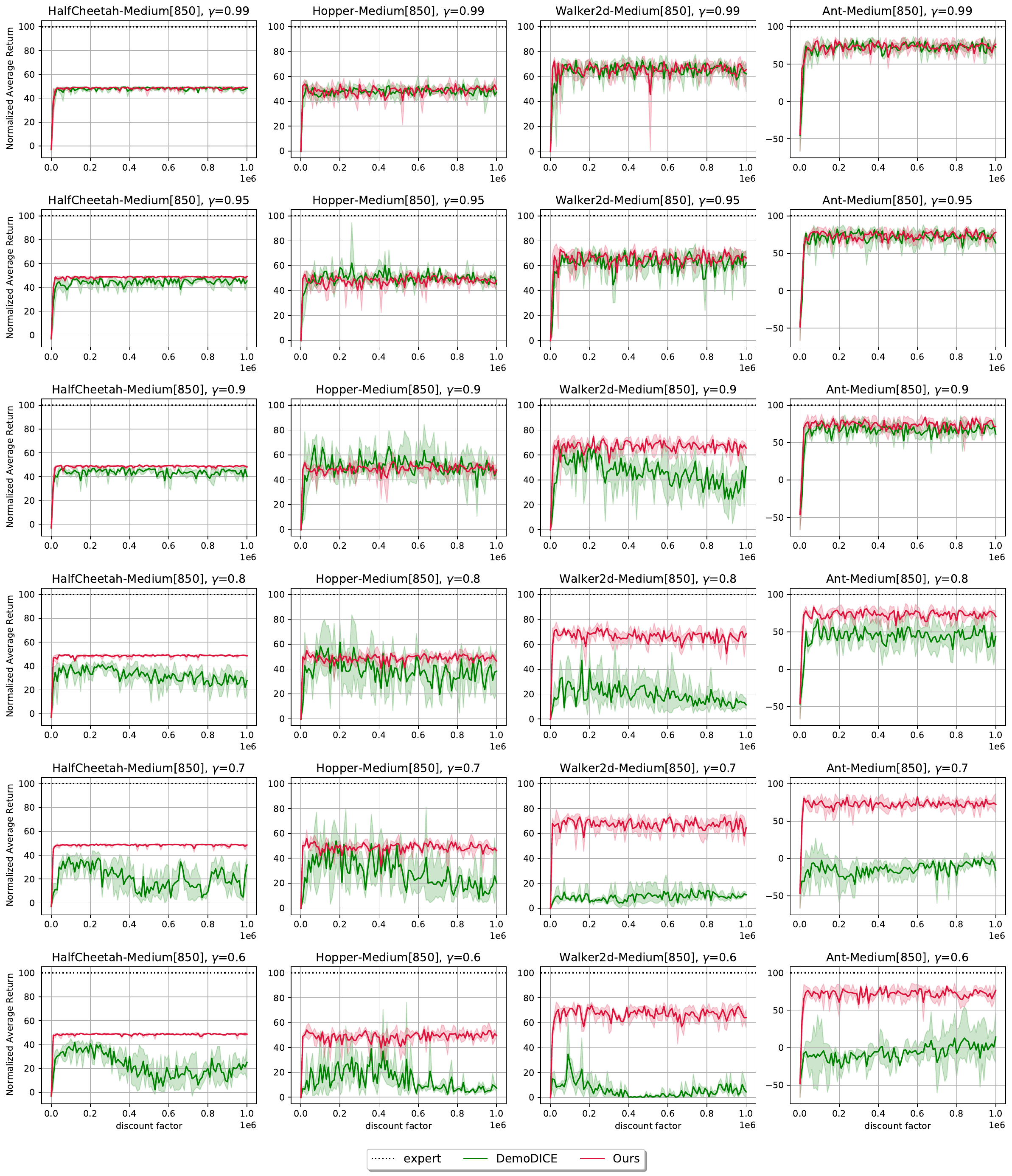}
    \caption{Performance of IGI and DemoDICE when suboptimal dataset $\mathcal{D}^O$ is made by 850 medium-level trajectories. Each column shows the normalized return of two algorithms according to the discount factor in the same environment.}
    \label{medium}
\end{figure}
\end{document}